\documentclass[twoside]{article}

\usepackage[accepted]{aistats2026}
\usepackage[round,authoryear]{natbib}

\usepackage{xcolor}
\usepackage{amsmath, amssymb}
\usepackage[linesnumbered,ruled,vlined]{algorithm2e}
\usepackage{amsthm}
\usepackage{thmtools}
\usepackage{verbatim}
\usepackage{hyperref}
\hypersetup{
    colorlinks=true,
    linkcolor={blue},
    citecolor={teal},
    urlcolor={teal}
}
\usepackage[nameinlink,capitalise]{cleveref}

\usepackage{graphicx}
\usepackage{array}

\usepackage{xcolor}
\usepackage{float}
\usepackage{inconsolata}
\usepackage{listings}
\usepackage{newfloat}

\definecolor{bg}{gray}{0.96}

\usepackage{listings}

\definecolor{bg}{gray}{0.96}

\begin{document}

\newcommand{\TODO}[1]{\noindent\setlength{\fboxsep}{3pt}\colorbox{red!50}{\textbf{TODO:} \textcolor{black}{\parbox{\dimexpr\linewidth-50pt}{\raggedright\noindent #1}}}}

\newcommand{\TODOsarah}[1]{\noindent\setlength{\fboxsep}{3pt}\colorbox{yellow!50}{\textbf{@Sarah:} \textcolor{black}{\parbox{\dimexpr\linewidth-50pt}{\raggedright\noindent #1}}}}

\newcommand{\E}{\mathbb{E}}
\renewcommand{\P}{\mathbb{P}}

\newcommand{\floor}[1]{\lfloor #1 \rfloor}
\newcommand{\ceil}[1]{\lceil #1 \rceil}

\newcommand{\R}{\mathbb{R}}
\newcommand{\N}{\mathbb{N}}
\newcommand{\Z}{\mathbb{Z}}
\newcommand{\Q}{\mathbb{Q}}
\newcommand{\C}{\mathbb{C}}
\newcommand{\F}{\mathcal{F}}
\renewcommand{\cal}[1]{\mathcal{#1}}

\newcommand{\ra}{\rightarrow}
\newcommand{\Ra}{\Rightarrow}
\newcommand{\la}{\leftarrow}
\newcommand{\La}{\Leftarrow}
\newcommand{\LRa}{\LeftRightarrow}
\newcommand{\lra}{\leftrightarrow}
\newcommand{\lRa}{\Leftrightarrow}

\newcommand{\var}{\mathrm{Var}}
\newcommand{\cov}{\mathrm{Cov}}
\newcommand{\indic}[1]{\mathbb{I}\{#1\}}
\newcommand{\abs}[1]{\left| #1 \right|}
\newcommand{\norm}[1]{\left\| #1 \right\|}
\newcommand{\set}[1]{\left\{ #1 \right\}}
\newcommand{\inner}[2]{\langle #1, #2 \rangle}
\newcommand{\qvar}[1]{\langle #1 \rangle}
\newcommand{\ip}[1]{\langle #1 \rangle}
\newcommand{\tens}{\mathbin{\mathop{\otimes}}}
\newcommand{\kron}{\otimes_K}

\newcommand{\pdiff}[2]{\frac{\partial #1}{\partial #2}}
\newcommand{\ddiff}[2]{\frac{\diff #1}{\diff #2}}
\newcommand{\grad}{\nabla}

\newcommand{\eps}{\varepsilon}
\newcommand{\del}{\partial}
\newcommand{\suchthat}{\;\middle|\;}
\newcommand{\st}{\;\text{s.t.}\;}

\newcommand{\limn}{\lim_{n \to \infty}}
\newcommand{\sumn}{\sum_{n=1}^\infty}
\newcommand{\prodn}{\prod_{n=1}^\infty}

\newcommand{\true}{\text{true}}
\newcommand{\false}{\text{false}}

\let\oldref\ref
\renewcommand{\ref}[1]{(\oldref{#1})}

\newcommand{\pa}[1]{\operatorname{pa}_G(#1) }
\newcommand{\ch}[1]{\operatorname{ch}_G(#1) }
\newcommand{\an}[1]{\operatorname{an}_G(#1) }
\newcommand{\de}[1]{\operatorname{de}_G(#1) }
\newcommand{\indepG}{\perp\!\!\perp}

\newcommand{\indep}{\perp\!\!\!\!\perp} 
\newcommand{\HSIC}{\operatorname{HSIC}}
\newcommand{\DO}[1]{\operatorname{do}(#1) }

\newcommand{\vect}[1]{\operatorname{vec}(#1) }


\crefformat{equation}{(#2#1#3)}  

\theoremstyle{plain}
\newtheorem{theorem}{Theorem}
\newtheorem{corollary}[theorem]{Corollary}
\newtheorem{lemma}[theorem]{Lemma}
\newtheorem{proposition}[theorem]{Proposition}
\theoremstyle{definition}
\newtheorem{example}[theorem]{Example}
\newtheorem{definition}[theorem]{Definition}  
\theoremstyle{remark}
\newtheorem{remark}[theorem]{Remark}
\newtheorem{viewpoint}[theorem]{Viewpoint}

\newcommand{\appendixtitle}[1]{%
  \hsize\textwidth
  \linewidth\hsize
  \toptitlebar
  {\centering {\Large\bfseries #1 \par}}
  \bottomtitlebar
}

%

%

\twocolumn[

\aistatstitle{Efficient Learning of Stationary Diffusions with Stein-type Discrepancies}

\aistatsauthor{
Fabian Bleile \And
Sarah Lumpp \And
Mathias Drton
}

\aistatsaddress{
Technical University of Munich\\Munich Center for Machine Learning \And
Technical University of Munich \And
Technical University of Munich\\Munich Center for Machine Learning
}

]

\begin{abstract}
  Learning a stationary diffusion amounts to estimating the parameters of a stochastic differential equation whose stationary distribution matches a target distribution. We build on the recently introduced kernel deviation from stationarity (KDS), which enforces stationarity by evaluating expectations of the diffusion's generator in a reproducing kernel Hilbert space. Leveraging the connection between KDS and Stein discrepancies, we introduce the Stein-type KDS (SKDS) as an alternative formulation. We prove that a vanishing SKDS guarantees alignment of the learned diffusion’s stationary distribution with the target. Furthermore, under broad parametrizations, SKDS is convex with an empirical version that is $\epsilon$-quasiconvex with high probability. Empirically, learning with SKDS attains comparable accuracy to KDS while substantially reducing computational cost, and yields improvements over the majority of competitive baselines.
\end{abstract}
\section{Introduction}
Understanding cause-and-effect relationships is fundamental across many scientific disciplines, from the social sciences to natural sciences and engineering. Causal modeling provides a principled framework for reasoning about such relationships, enabling predictions under interventions and guiding decision-making in complex systems. Traditional approaches to causal inference, often built on structural causal models (SCMs), assume an underlying directed acyclic graph (DAG) that encodes causal dependencies. While this assumption facilitates theoretical analysis and algorithmic implementation, it poses challenges in systems with cyclic dependencies \citep{Spirtes:2000,Pearl:2009}.

A causal model is designed to represent observational, interventional, and sometimes counterfactual distributions.
Stationary Itô diffusions, modeled through stochastic differential equations (SDEs), provide a natural framework for representing causal relationships in systems where cyclic dependencies arise \citep{Lorch:2024, varando2020graphical,hansen2014causal}. Unlike traditional SCMs, this approach inherently accommodates cyclic dependencies while ensuring a well-defined stationary distribution that characterizes long-term system behavior. In the causal framework of stationary diffusions, the observational distribution of a $d$-dimensional random variable $X$ is modeled as the stationary distribution $P$ of a stochastic process $(X_t)_{t \geq 0}$, defined by the following SDE
\begin{equation}\label{eq:SDE}
    dX_t = b(X_t)dt + \sigma(X_t)dB_t,
\end{equation}
with drift function $b$ and diffusion function $\sigma$.
Intuitively, the drift term $b$ describes the deterministic tendency of the process, while the diffusion term $\sigma$ models random fluctuations driven by Brownian motion.
The distribution $P$ is called stationary if $(X_t)_{t \geq 0}$ solves \cref{eq:SDE} with $X_t \sim P$ for every $t \geq 0$. 
Under mild conditions ensuring that the process is stable and does not diverge, such a stationary solution exists and is unique \citep{huang2015steady}.

Let $p$ be a probability density, then the Langevin diffusion $dX_t = \nabla \log p(X_t) dt + \sqrt{2}I_d dB_t$ is a stationary diffusion with stationary density $p$. This reduces the task of finding a causal model, as described, to a score estimation problem. However it is non-trivial how to incorporate interventional distributions into this model. Interventions in this causal framework are modeled by modifying the drift and diffusion functions. The interventional distribution is then, assuming existence and uniqueness, the respective stationary distribution \citep{hansen2014causal, Lorch:2024}.
In contrast to traditional SCMs, which refer to structural equations among the observed variables directly, the stationary diffusion causal framework focuses on joint distributions that emerge from a (possibly intervened) system's equilibrium behavior.

Consider a model for the SDE in~\cref{eq:SDE} that treats drift and diffusion terms as unknown parameters. Learning a stationary diffusion means adjusting its parameters so that a stationary solution~$P$ exists and matches an observed target distribution~$\mu$.
\Citet{Lorch:2024} propose a loss---\emph{Kernel Deviation from Stationarity} (KDS)---that measures the discrepancy between~$\mu$ and~$P$ using only samples of~$\mu$ and the parameters of~$b$ and~$\sigma$. This enables direct optimization of the SDE parameters.
This approach performs well in small simulations but is computationally expensive and, thus, not feasible for data of larger dimension. 
Our main contribution is a novel loss function for learning stationary SDEs---the \emph{Stein-type Kernel Deviation from Stationarity (SKDS)}, which offers significantly lower computational cost at similar estimation performance.
We motivate SKDS from Stein's method, inspired by the work on Stein discrepancies \citep{gorham2015measuring, liu2016kernelized, gorham2017measuring}, and show its practical and theoretical advantages over KDS in learning from interventional data. 
Before introducing SKDS, we review relevant results from kernel theory, Stein discrepancies, and the martingale characterization of stationarity.

\section{Background}\label{sec:Background}

We denote the space of functions from a domain \(U\) to a vector space \(V\) by  \(\Gamma(U, V)\); and we write \(\Gamma(U) := \Gamma(U, \mathbb{R})\) for real-valued functions.
We define $C^n(U, V)$ (respectively, $C_b^n(U, V)$) as the set of $n$-times continuously differentiable functions \(f \in \Gamma(U, V)\) (respectively, $n$-times continuously differentiable with uniformly bounded differentials). If the target space is clear, we write simply $C^n(U)$ and $C_b^n(U)$.
For a differentiable $f \in \Gamma(\R^n, \R^n)$, we write the divergence of $f$ as $\inner{\grad_x}{f(x)}=\operatorname{div}f(x) = \sum_{i \in [n]} \partial_i f_i(x)$.

\paragraph{Kernel Theory.}\label{subsec:Kernel Theory}
Throughout, let $k : \R^d \times \R^d \ra \R$ be a positive definite kernel function, and let $K : \R^d \times \R^d \ra \R^{d\times d}$ be a positive definite matrix-valued kernel function. Let $\cal{H}$ be the reproducing kernel Hilbert space (RKHS) defined by $k$, and let $\cal{H}^d$ be the respective RKHS defined by $K$. The reproducing property in  $\cal{H}^d$ requires that for $f \in \cal{H}^d$ and $x \in \R^d$, we have
\begin{equation}\label{eq:R^d valued kernel reproducing property main}
    \inner{f(x)}{v}_{\R^d} = \inner{f}{K(\cdot, x)v}_{\cal{H}^d} \quad \text{for all }v\in \R^d.
\end{equation}
Specifically, we will be interested in the case when $K = k I_d$. Then, it holds that $f \in \cal{H}^d$ if and only if $f_i \in \cal{H}$ for all $i \in [d]$ (see \cref{app:Kernel Theory}).

\paragraph{Characterization of Stationarity.}
The generator $\cal{L}$ of a stochastic process defined by \cref{eq:SDE} is a linear operator describing the infinitesimal expected change of functions $f \in C_b(\R^d)$ along the process.  
For a sufficiently regular SDE \cref{eq:SDE}, $\cal{L}$ acts as a differential operator on $f \in C^2(\R^d)$ and is given by
\begin{equation}\label{eq:differential operator}
    \cal{L} f(x) = \inner{b(x)}{\nabla f(x)} + \frac{1}{2} \operatorname{tr}\big(a(x) \nabla \nabla f(x)\big),
\end{equation}
where $a(x) = \sigma(x)\sigma(x)^T$.  
A probability density $\mu$ is a stationary solution of the SDE \cref{eq:SDE} if and only if it satisfies the martingale problem,
\begin{equation}\label{eq:simplified martingale problem}
    \E_{X \sim \mu}[\cal{L} f(X)] = 0 \quad \text{for all } f \in C^2(\R^d).
\end{equation}

The class of $\mu$-targeted diffusions, i.e., diffusion processes defined by \cref{eq:SDE} with stationary density $\mu$, admits a full characterization: the drift decomposes as
\begin{equation}\label{eq:drift mu targeted}
    b(x) = \frac{1}{2 \mu(x)} \inner{\grad}{\mu(x) \left(a(x) + c(x)\right)},
\end{equation}
with a skew-symmetric stream coefficient $c$ \citep{gorham2019measuring}. In this case, the generator can be written as
\begin{equation}\label{eq:main text generator density version}
    (\cal{L} f)(x) = \frac{1}{2 \mu(x)}\inner{\grad}{\mu(x) (a(x) + c(x)) \nabla f(x)}.
\end{equation}
A stationary diffusion process is called reversible if its law is invariant under time reversal. This holds if and only if detailed balance holds, i.e., $\mu^{-1} \inner{\grad}{\mu c} = 0$ \citep[Prop.~4.5.]{pavliotis2014stochastic}. This term is an additive part of the drift in \cref{eq:drift mu targeted}, and we refer to it as the non-reversible part of the drift.
We refer to the supplement \Cref{app:SDE-Generator-Martingale Problem} for the derivation and technical conditions.

Having characterized stationarity in terms of the generator $\cal{L}$, we now turn to how this condition can be used as a learning criterion.

\paragraph{Identifiability.}\label{subsec:Identifiability}
The parameters of a stationary diffusion cannot be uniquely identified from its stationary density $\mu$ alone. In particular, an SDE of the form $dX_t = sb(X_t)dt + \sqrt{s}\sigma(X_t)dB_t$ has generator $s \mathcal{L}$ and, if a stationary distribution exists, it coincides with that of the original (unscaled) SDE. Hence, at best, $\mathcal{L}$ is identifiable from $\mu$ up to a multiplicative constant, and moreover, only through the product $\sigma \sigma^T$ \citep[Example 5.5]{hansen2014causal}. To resolve this ambiguity in our setting, we fix the speed parameter and model $\sigma \sigma^T$ directly, rather than $\sigma$ itself. For further results on identifiability of SDE parameters we refer to the literature \citep{hansen2014causal, wang2024generator, dettling2023identifiability}.

\paragraph{Kernel Deviation from Stationarity \citep[KDS]{Lorch:2024}.}
\label{subsec:Kernel Deviation from Stationarity}
The KDS is defined via the supremum in the martingale problem \cref{eq:simplified martingale problem} over the unit ball $\cal{H}_{\leq 1}$ of a reproducing kernel Hilbert space (RKHS) $\mathcal{H}$, yielding
\begin{equation}\label{eq:KDS}
    \sqrt{\operatorname{KDS}(\cal{L}, \mu; \cal{F})} = \sup_{f \in \cal{H}_{\leq 1}} \mathbb{E}_{X \sim \mu}[\mathcal{L}f(X)].
\end{equation}
This expression is meaningful when $\mathcal{H}$ is chosen such that $\operatorname{core}(\mathcal{A}) \subseteq \mathcal{H} \subseteq C^2(\R^d)$, ensuring that the operator $\mathcal{L}$ acts appropriately on elements of $\mathcal{H}$. The \emph{core} is an inclusion-minimal set of functions that guarantees the equivalence between $\mu$ being a stationary solution of \cref{eq:SDE} and of \cref{eq:simplified martingale problem}. For details we refer to the supplement \cref{app:SDE-Generator-Martingale Problem}.

The \emph{Stein-type kernel deviation from stationarity} (SKDS), introduced as our proposed loss function in \cref{sec:SKDS}, builds on the same principles but replaces $\nabla f$ in \cref{eq:differential operator} by a function $g$ in some $\cal{H}^d$. This generalization is inspired by the theory of Stein discrepancies, in particular the diffusion kernel Stein discrepancy (DKSD) introduced by \Citet{barp2019minimum}.

\paragraph{Stein Discrepancy.}\label{subsec:Stein Discrepancy}
A Stein discrepancy compares two distributions $Q$ and $P$ by measuring deviations that vanish under $P$. Concretely, given a Stein operator $\cal{T}_P$ such that $\E_{X \sim P}[(\cal{T}_Pg)(X)] = 0$ for all $g$ in a suitable function class $\cal{G}$, one defines
\begin{equation}\label{eq:Stein Discrepancy Generator Approach}
    S(Q, P; \cal{G}; \cal{T}_P) := \sup_{g \in \cal{G}} \abs{\mathbb{E}_{X \sim Q}[(\cal{T}_Pg)(X)]}.
\end{equation}
This yields an integral probability metric (IPM) tailored to $P$, often more tractable than generic IPMs \citep{gorham2015measuring, anastasiou2023stein}.  

A particularly useful choice arises when $P$ is the stationary distribution of an SDE with generator $\mathcal{L}$. In this case, $\mathcal{T}_P = \mathcal{L}$, and the Stein identity $\E_{X \sim P}[(\mathcal{T}_P g)(X)] = 0$ coincides with the martingale problem in \cref{eq:simplified martingale problem}. This generator-based perspective forms the foundation of kernel Stein discrepancies and their diffusion variants, which we build on next.

\section{Stein-Type Kernel Deviation from Stationarity}\label{sec:SKDS}
Building on the generator approach, we obtain from \cref{eq:main text generator density version} the (first-order) \emph{diffusion Stein operator} \citep{gorham2019measuring}
\begin{equation}\label{eq:diffusion stein operator}
    \mathcal{D}^m_\mu g := \frac{1}{\mu}\,\inner{\nabla}{\mu m g},
\end{equation}
where $g \in \Gamma(\R^d,\R^d)$, $m \in \Gamma(\R^d,\R^{d\times d})$ and $\mu$ is a $C^1$-density of $P$.  
Together with a kernel $K$, this yields the \emph{diffusion kernel Stein discrepancy} \citep{barp2019minimum}.  

In our setting, we have the matrix $m = a+c$, with $a$ and $c$ as in \cref{sec:Background}. For $g \in C^1(\R^d,\R^d)$, we rewrite
\begin{equation}\label{eq:DS operator and SKDS operator}
    \begin{aligned}
    \mathcal{D}^{a+c}_\mu g  &= \frac{1}{\mu} \inner{\nabla}{\mu(a + c)g}\\
    &= \frac{1}{\mu} \left( \inner{\inner{\nabla}{\mu(a + c)}}{g} + \operatorname{tr}\left((a + c) \nabla g \right)\right)\\
    &= \inner{2b}{g} + \operatorname{tr}\left(a\nabla g\right) + \operatorname{tr}\left(c\nabla g\right),
\end{aligned}
\end{equation}
using standard vector calculus (see \cref{app:Vector Calculus}).

Note that the additional term $\operatorname{tr}(c\nabla g)$ sets \cref{eq:DS operator and SKDS operator} apart from the diffusion generator 
\(
    \mathcal{L}f = \inner{b}{\nabla f} + \tfrac12 \operatorname{tr}(a \nabla^2 f).
\) 
Indeed, for any $f \in C^2(\R^d)$, one has 
\(
    \mathcal{D}^{a+c}_\mu \nabla f = 2\mathcal{L}f,
\) 
since $\operatorname{tr}(c\nabla \nabla f) = 0$.\footnote{In fact, $\operatorname{tr}(CB) = 0$ for any symmetric matrix $B$ and skew-symmetric matrix $C$, because 
$\operatorname{tr}(CB) = \operatorname{tr}((CB)^T) = -\operatorname{tr}(BC) = -\operatorname{tr}(CB)$.}

In contrast to $b$ and $\sigma$, the skew-symmetric component $c$ is not directly accessible during learning---similar to how the stationary density of the diffusion specified by $b$ and $\sigma$ is unknown.
While $\mu$ and $c$ are both not easily computed, they exist uniquely. In particular, $c$ is uniquely determined by $b$ and $\sigma$ up to $c \nabla \log \mu + \langle \nabla , c \rangle$. The density $\mu$ solves the stationary Fokker-Planck equation \citep{huang2015steady}; and the skew-symmetric matrix $c$ is then defined via \cref{eq:drift mu targeted}.
Dropping $\operatorname{tr}(c\nabla g)$ from $\mathcal{D}^{a+c}_\mu g$, or equivalently replacing $\nabla f$ by $g$ in $\mathcal{L}$, motivates the following operator.

\begin{definition}[SKDS Operator]\label{def:SKDS Operator}
    Let $b \in C^0(\R^d,\R^d)$ and $\sigma \in C^0(\R^d,\R^{d \times d})$.  
    For $g \in C^1(\R^d,\R^d)$ define
    \[
        \mathcal{S}(g)(x) := 2\inner{b(x)}{g(x)} + \operatorname{tr}\!\left(\sigma(x)\sigma(x)^\top \nabla_x g(x)\right).
    \]
    For a matrix-valued function $A \in C^1(\R^d,\R^{d \times d})$ we extend column-wise,
    \[
        (\mathcal{S}(A)(x))_i := \mathcal{S}(A_i)(x),
    \]
    where $A_i$ is the $i$-th column of $A$.  
    We write $\mathcal{S}_P$ (respectively,  $\mathcal{S}_p$) if the operator is associated with an SDE inducing a stationary distribution $P$ (respectively, stationary density $p$).
\end{definition}

\noindent The above operator gives rise to the \emph{Stein-type kernel deviation from stationarity} (SKDS).

\begin{definition}[Stein-type KDS] \label{def:Stein-type Kernel Deviation from Stationarity}
Let $(X_t)_{t \geq 0}$ be a diffusion with drift $b \in \Gamma(\R^d,\R^d)$ and diffusion term $\sigma \in \Gamma(\R^d,\R^{d \times d})$, and let $\mathcal{S}$ be the associated SKDS operator.  
For a scalar-valued kernel $k$ with RKHS $\mathcal{H}$, let $\mathcal{H}^d$ be the RKHS associated with $K = kI_d$ with unit ball $\mathcal{H}^d_{\leq 1}$.  
For a distribution $\mu$, the SKDS is defined as
\begin{equation}\label{eq:SKDS definition}
    \sqrt{\operatorname{SKDS}(\mathcal{S}, \mu; \mathcal{H}^d_{\leq 1})}
    := \sup_{g \in \mathcal{H}^d_{\leq 1}} \E_{X \sim \mu}[\mathcal{S}g(X)].
\end{equation}
\end{definition}

Following the approach of \citet[Example Section 3.3]{Lorch:2024}, we provide a simple illustration of how gradient descent on the Stein-type KDS can be used to infer the correct model parameters.

\begin{example}[SKDS as a learning objective]\label{ex:SKDS as learning obj}
Consider the one-dimensional linear SDE
\( dX_t = b(x-\alpha)\,dt + \sigma\,dB_t \),
which has stationary distribution
\( \mathcal{N}(\alpha, -\tfrac{\sigma^2}{2b}) \) whenever $b<0$, $\sigma>0$~\citep{jacobsen1993brief}.  
Suppose the target distribution is $\mathcal{N}(1, \tfrac12)$, corresponding to the stationary solution of the SDE
\( dX_t = -4(x-1)\,dt + 2\,dB_t \).  
For illustration, in \cref{fig:SKDS as learning objective}, we compare to two alternatives:  
(red) $\alpha=\tfrac14,\ b=-4,\ \sigma=2$, and  
(blue) $\alpha=1,\ b=-4,\ \sigma=1$.  
\cref{fig:SKDS as learning objective} shows that SKDS correctly identifies the true parameters by attaining vanishing partial derivatives at the optimum.
\end{example}
\begin{figure*}
    \centering
    \includegraphics[width=\textwidth]{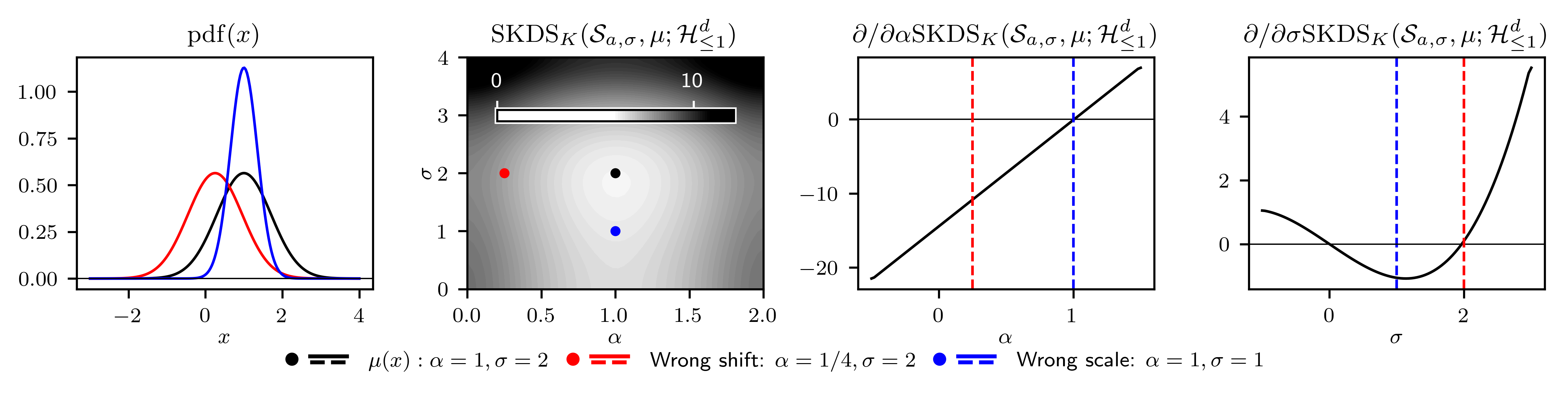}
    \caption{Stein-Type KDS for a stationary linear SDE: We applied SKDS to $n=5000$ samples from the target distribution $\mu$, using a one-dimensional Gaussian kernel with bandwidth $0.5$. The models correspond to the SDE $dX_t = -4(X_t - \alpha)dt + \sigma dB_t$ for different values of $\alpha$ and $\sigma$. 1) PDFs of the stationary distributions.
    2) An SKDS contour plot in $\alpha$ and $\sigma$.
    3) and 4) Partial derivatives of the SKDS objective function for both model alternatives, showing that the gradient vanishes at the parameter values of the ground-truth model.} \label{fig:SKDS as learning objective}
\end{figure*}

\subsection{Closed Form Representation}
For sufficiently regular drift, diffusion, and kernel functions, the  function in \cref{eq:SKDS definition} can be expressed as an inner product in $\cal{H}^d$ with a representer function $g_{\cal{S}, \mu} \in \cal{H}^d$.
To state the fact, we introduce some notation. For a differentiable kernel $k(x,y)$, we write $\partial_{x_j} k(x,x)$ for the derivative with respect to the $j$-th entry of both arguments.
We denote by $\mathcal{S}_i k(x,y)$, $i=\{1,2\}$, the SKDS operator applied to the $i$-th argument. 
Finally, $L^2_p(\R^m, \R^n)$ is the space of measurable functions $f:\R^m \to \R^n$ with $\|f\|_{L^2_p}^2 := \int \|f(x)\|^2 \, \mathrm{d}p(x) < \infty$, for a probability measure $p$ on $\R^m$.

\begin{lemma}[Representer function $g_{\cal{S}, \mu}$]
\label{lem:Representer function SKDS}
    Let $\mu$ be a probability density. Let $b$ and $\sigma$ be $\mu$-square-integrable, i.e., in $L^2_\mu$. Let $k(x,x)$ be continuously differentiable with $k$ and $\partial_{x_j} k(x,x)$ $\mu$-integrable. 
    Then there exists a unique representer function $g_{\mathcal{S}, \mu} \in \cal{H}^d$ such that
    \begin{equation}\label{eq:SKDS representation function}
        \E_{X \sim \mu}[\mathcal{S} g(X)] = \inner{g}{g_{\mathcal{S}, \mu}}_{\cal{H}^d}
    \end{equation}
    with explicit form $g_{\mathcal{S}, \mu}(\cdot) = \E_{X \sim \mu}[\mathcal{S}_1 K(X, \cdot)]$. 
\end{lemma}
The proof of this and later results in this paper are in the supplement (\cref{Supp:PROOFS}).
With the representer function and assumptions on the regularity of the drift and diffusion functions, SDKS admits a closed form.
\begin{lemma}[SKDS Closed Form]
    \label{lem:SKDS closed form}
    We adopt the assumptions of \cref{lem:Representer function SKDS} and additionally assume bounded $b$, $\sigma$, $k$ and $\partial_{x_i, y_j} k(x,y)$. Moreover, let $k$ be twice continuously differentiable. Then,
    \begin{equation}\label{eq:SKDS=norm}
        \operatorname{SKDS}(\mathcal{S}, \mu; \cal{H}^d_{\leq 1}) = \norm{g_{\mathcal{S}, \mu}}_{\cal{H}^d}^2,
    \end{equation}
    where the norm takes the following closed expression
    \begin{equation}\label{eq:SKDS-closed-form}
        \norm{g_{\mathcal{S}, \mu}}_{\cal{H}^d}^2 = \E_{X \sim \mu}\left[\E_{Y \sim \mu}\left[ \mathcal{S}_1 \mathcal{S}_2 K(X, Y)\right] \right].
    \end{equation}
\end{lemma}

Note that bounded functions $b$ and $\sigma$ ensure the boundedness of the operator $\mathcal{S}$, which is a sufficient condition to interchange $\mathcal{S}$ and the integral to receive the RHS in \cref{eq:SKDS-closed-form}.


For a matrix-valued kernel of the form $K = k I_d$, we can simplify \cref{eq:SKDS-closed-form} even further as
    \begin{equation} \label{eq:simple closed form}
        \begin{aligned}
            \mathcal{S}_1 \mathcal{S}_2 K(x, y) =\quad& 4 \inner{b(x)}{b(y) k(x,y)}\\
    \quad +& 2 \inner{b(x)}{a(y)^T \nabla_y k(x, y)}\\
    \quad +& 2 \inner{b(y)}{a(x)^T \nabla_x k(x, y)}\\
    \quad +& \operatorname{tr}\left(a(x) a(y)^T \nabla_x \nabla_y k(x, y)\right).
        \end{aligned}
    \end{equation}

\subsection{Properties}
We define $C^{n,n}$ (respectively, $C_b^{n,n}$) to be the set of functions $k \in \Gamma(\R^{n} \times \R^{n})$ with $(x,y) \mapsto \grad_x^l \grad_x^l k(x,y)$ continuous (respectively, continuous and uniformly bounded) for all $l \in [n]$. 
The Wasserstein-$2$ distance is defined as $d_{\mathcal{W}_2}(\mu, p) = \inf_{X \sim \mu, Z \sim P} \E[\norm{X - Z}_2]$.

The following result is inspired by the diffusion kernel Stein discrepancy, which can be upper bounded by the Wasserstein-$2$ distance \citep[Proposition 9]{gorham2017measuring}. In the respective result on the SKDS, an additional term dependent on the stream-coefficient $c$ appears, which is non-accessible during learning.

\begin{proposition}\label{prop:SKDS upper bound}
    Let $\mathcal{S}_p$ be the SKDS operator associated to an SDE, $dX_t = b(X_t)dt + \sigma(X_t)dB_t$, with stationary density $p$ supported on all of $\R^d$. Assume the following: (i) target probability density $\mu \in C^1$ supported on all of $\R^d$; (ii) $b \in C_b^1(\R^d, \R^d)$, and $a$, $c \in C_b^1(\R^d, \R^{d \times d})$; (iii) $K = k I_d$ with $k \in C_b^{(2,2)}(\R^d \times \R^d, \R)$. Then 
    \begin{multline*}
        \label{eq:SKDS x Wasserstein}
            \sqrt{\mathrm{SKDS}(\mathcal{S}_p, \mu; \cal{H}^d_{\leq 1})}
            \leq  \lambda_k \lambda_{b, \sigma} \gamma({d_{\mathcal{W}_2}(\mu, p)})\\
            + \norm{k}_\infty \norm{c \grad \log p + \inner{\grad}{c}}_{L^2_p(\R^d, \R^d)},
    \end{multline*}
    where $\gamma(x) = x + \sqrt{x}$, and $\lambda_k$, $\lambda_{b, \sigma}$ are constants depending on the regularity of $k$, $b$ and $\sigma$.
\end{proposition}
The additional term $c \grad \log p + \inner{\nabla}{c}$ arises from the skew-symmetric stream coefficient \(c\) in the drift decomposition \cref{eq:drift mu targeted}. It captures precisely the non-reversible part of the drift: when this term is zero,  detailed balance  holds and the diffusion is reversible (cf.~\cref{eq:non-reversible drift}). Note that the stream coefficient $c$ is not a free parameter in the last result. Under these assumptions it is uniquely determined by $b$ and $\sigma$ up to $c \nabla \log \mu + \langle \nabla , c \rangle$. Hence, $c$ is a property of the SDE defined by $b$ and $\sigma$.

The following result shows that the SKDS vanishes if and only if $\cal{S}$ is associated to a reversible diffusion with stationary density $\mu$. In particular, vanishing SKDS implies that the stationary density of the corresponding diffusion matches the target density. 
\begin{theorem}[SKDS characterizes reversible diffusions]
\label{th:SKDS consistency}
    Let $\mathcal{S}$ be the SKDS operator associated to SDE $dX_t = b(X_t)dt + \sigma(X_t)dB_t$.
    Adopt the assumptions (i), (ii) and (iii) of \cref{prop:SKDS upper bound}. Moreover, assume (iv) $\sigma(x)\sigma(x)^T$ is positive definite for all $x \in \mathbb{R}^d$; (v) $C_c^\infty \subseteq \mathcal{H}$. Then $(X_t)_{t \geq 0}$ is a stationary reversible diffusion process with stationary density $\mu$ if and only if $\operatorname{SKDS}(\mathcal{S}, \mu; \cal{H}^d_{\leq 1}) = 0$.
\end{theorem}

The assumptions on $b$ and $\sigma$--namely (ii) and (iv)--ensure that $C_c^\infty(\mathbb{R}^d)$ forms a \emph{core} of the generator $\mathcal{L}$ \citep{ethier2009markov}. For the proof of \cref{th:SKDS consistency}, it is crucial that $\mathcal{H}^d$ contains this core. It is possible that more general core results could relax the required regularity assumptions on the drift and diffusion while still yielding \cref{th:SKDS consistency}. Furthermore, \citet{Lorch:2024} argue in a related setting that, in practice, the conclusion extends informally to the case of merely continuous drift and diffusion functions.  Continuous functions can become unbounded only as $\|x\|\to\infty$, whereas our studied process remains stationary.

Note that, up to this point, we have considered only a single target distribution. When applying SKDS to learn a causal framework, however, a parametrized model is trained jointly on multiple datasets -- both observational and interventional (see \cref{alg:SKDS learning}). In this joint setting, even a sufficiently expressive parametrized model need not admit a reversible solution. Despite a lack of theoretical guarantees characterizing the nature of a global solution in the case of multiple datasets, our approach as well as the approach by \citet{Lorch:2024} perform well in experiments.

\section{Learning Stationary Diffusions}\label{sec:Learning Stationary Diffusions}
\subsection{Parametrization}\label{subsec:Stein-type KDS Parametrization}

Consider a parametrized SDE of form \cref{eq:SDE} with drift function $b_{\theta}$ and diffusion function $\sigma_{\theta}$ depending on parameters $\theta$.
Learning the parameters of a stationary diffusion in this setting means adjusting $\theta$ so that a stationary solution with distribution $P$ exists and matches an observed target distribution $\mu$. The associated SKDS operator is denoted by $\mathcal{S}^{\theta}$. The following result and definition show the SKDS to be convex for a rich modeling class of drift and diffusion functions (for details see \Cref{Supp:PROOFS}).

\begin{corollary}[Convexity of SKDS]\label{cor:SKDS convex}
    Let $\cal{T} : C^1(\R^d, \R^d) \ra C^0(\R^d, \R^l)$ be a linear operator independent of the parametrization $\theta \in \R^l$, $l \in \N$. Assume that the parametrized SKDS operator $\mathcal{S}^\theta$ takes the form
    \begin{equation}\label{eq:lin operator}
        \mathcal{S}^{\theta}\cdot = \inner{\theta}{\cal{T}\cdot}.
    \end{equation}
    Then $\operatorname{SKDS}(\mathcal{S}^\theta, \mu; \cal{H}^d_{\leq 1})$ is convex in $\theta$.
\end{corollary}

The class of SDEs which admit a linear representation of the SKDS operator is notably larger than affine linear SDEs (which define, e.g., Ornstein-Uhlenbeck processes). We introduce a parameterized class of SDEs for which the SKDS operator becomes linear in the model parameters.

\begin{example}[Linear SDE parametrization]\label{ex:linear SDE parametrization}
Let $\cal{J} = \set{j_i : \R^d \ra \R}_{i \in [l]}$ be a basis of a finite-dimensional function space, e.g., a set of polynomial basis functions. Define the feature map $j(x) = (j_1(x), \dots, j_l(x))^T \in \mathbb{R}^l$, and let $B^\theta \in \mathbb{R}^{d \times l}$ be a parameter matrix. The drift is then modeled as
\[
b(x) = B^\theta j(x).
\]
To preserve linearity in the parameters, we model the \emph{diffusion} via the symmetric positive semi-definite (PSD) matrix function $a(x) = \sigma(x)\sigma(x)^T$ directly (recall the identifiability remarks in \cref{subsec:Identifiability}). Let $\cal{V} = \{v_i\}_{i \in [m]} \subset \Gamma(\mathbb{R}^d, \mathbb{R}^d)$ be a set of vector-valued functions, and consider the set of rank-one PSD matrices $\{v_i(x)v_i(x)^T\}_{i \in [m]}$. Then, write $a(x)$ as a nonnegative combination with parameters $A_i^\theta \geq 0$ as
\[
a(x) = \sum_{i=1}^m A_i^\theta \, v_i(x)v_i(x)^T.
\]
This ensures that $a(x)$ is PSD by construction. The SKDS operator takes then the form from \cref{eq:lin operator} for
\begin{equation}\label{eq:lin operator T main text}
    \mathcal{T}g(x) = 
    \begin{pmatrix}
    \operatorname{vec}\big(j(x) \otimes g(x)\big) \\
    V(x) \operatorname{vec}\big(\nabla g(x)\big)
    \end{pmatrix} \in \Gamma(\R^d, \R^{ld + m}),
\end{equation}
where $\otimes$ denotes the outer product, and $V$ stacks the vectorized basis diffusion matrices, i.e., $V(x) = (\vect{v_1(x)v_1(x)^T}\mid \ldots \mid  \vect{v_m(x)v_m(x)^T})^T$. For the calculations we refer to \cref{app:Calculations for Example 10}.

The representational capacity of the above parameterized class is governed by the choices of basis functions in $\cal{J}$ and $\cal{V}$, which define the model’s expressiveness. The dimensions $l$ and $m$ determine the number of parameters ($dl + m$).
\end{example}

\subsection{Stein-Type KDS as a Learning Objective}\label{subsec:Stein-type KDS as a learning objective}

As the exact SKDS requires evaluation over the full support of $\mu$, in applications, we turn to an empirical estimate from discrete samples. Let $D = \set{x_1, \ldots, x_N}$ be an i.i.d. sample from the target distribution $\mu$. An example for an unbiased empirical estimate of \cref{eq:SKDS-closed-form} is
\begin{equation}\label{eq:empirical estimate SKDS}
    \hat{\operatorname{SKDS}}(\mathcal{S}^{\theta}, D; \cal{H}^d_{\leq 1}) = \frac{1}{\floor{N/2}} \sum_{n = 1}^{\floor{N/2}} \mathcal{S}_1^{\theta} \mathcal{S}_2^{\theta} K(x_{2n - 1},x_{2n}).
\end{equation}
The computation of this estimator scales linearly with $N$. It is unbiased by linearity of expectation and the i.i.d.\ sampling. We specifically consider this linear statistic over the also possible U-statistic because its independent summands enable the simple bounds in \cref{prop:SKDS emp convex}. For benchmarking in \cref{sec:Synthetic Experiments}, we also employ this linear estimator in the KDS-based approach proposed by Lorch.

By \cref{cor:SKDS convex}, the SKDS is convex in terms of a certain linear parametrization of the SDE. The following result shows that the empirical estimate \cref{eq:empirical estimate SKDS} itself is `almost convex' with high probability, depending on the regularity of the drift and diffusion function.

\begin{proposition}[SKDS empirical estimate convex]\label{prop:SKDS emp convex}
    We build on the setup from \cref{ex:linear SDE parametrization} and continue under the assumptions of \cref{prop:SKDS upper bound}. In addition, we assume that the basis functions are continuously differentiable and uniformly bounded, i.e., $j$, $v_i \in C_b^1$ with
    \[
    \|j(x)\|_\infty, \ \|v_i(x)v_i(x)^T\|_\infty \leq C \quad \forall i \in [m],
    \]
    for some constant $C \in [1,\infty)$ and all $i \in [m]$. This ensures that the parametrized drift and diffusion satisfy the conditions of \cref{prop:SKDS upper bound}. Under these assumptions, the empirical estimate $\hat{\operatorname{SKDS}}(\mathcal{S}^{\theta}, D; \cal{H}^d_{\leq 1})$ defined as in \cref{eq:empirical estimate SKDS} is $\varepsilon$-strongly quasiconvex with high probability. In other words, 
    \begin{equation}
        \hat{\operatorname{SKDS}}(\mathcal{S}^{\theta}, D; \cal{H}^d_{\leq 1}) + \varepsilon \norm{\theta}_2
    \end{equation}
    is convex in $\theta$ with probability greater than
    \begin{equation}
        1 - 2 (dl + m) \exp{\left( - \frac{\floor{N/2} \varepsilon^2 / 2}{L_2 + 2 L_1 \varepsilon/3}\right)},
    \end{equation}
    where $L_1$ and $L_2$ are constants depending on the bounds on $j$, $v$ and the kernel $k$.
\end{proposition}

We remark that $d$ enters the constants $L_1$ and $L_2$ with order $O(d^7)$ (see \cref{eq:L_1} and \cref{eq:L_2}). Consequently, for the lower bound in the previous proposition to remain meaningful, the number of samples $N$ must scale on the order of $\Theta(d^8)$ in the dimensionality of the data.

Note that as $\theta \to 0$, the SKDS vanishes. This behavior is closely related to the speed scaling ambiguity in stochastic differential equations (see \cref{subsec:Identifiability}). To avoid this degeneracy and rule out the trivial minimizer $\theta = 0$, one can restrict the parameter space to the convex set $\set{\theta \in \R^{dl + m} \mid \theta_1 = \alpha}$ for some fixed constant $\alpha$. The choice of $\alpha$ should be informed by domain knowledge: it must be non-zero and its sign must be appropriate for the problem at hand. This type of constraint is consistent with previous work, e.g., by fixing the mean reversion \citep[D.4]{Lorch:2024}.

\subsection{Comparison with KDS}
We have $\mathcal{L} u = \frac{1}{2} \mathcal{S}[\nabla u]$, which allows an analogous line of reasoning for both objectives. A comparison of their definitions in terms of operators, RKHSs, and closed-form expressions is given in \cref{tab:kds_vs_skds}. Notably, the {KDS} vanishes for Itô diffusions with stationary density $\mu$, whereas the {SKDS} vanishes for those that are additionally reversible (\cref{th:SKDS consistency}).

\begin{table*}[t]
    \centering
    \small
    \setlength{\tabcolsep}{3pt} 
    \renewcommand{\arraystretch}{1.7} 
    \begin{tabular}{|c|c|c|}
    \hline
    & \textbf{KDS} & \textbf{Stein-type KDS} \\
    \hline
    Operator &$\begin{aligned}
    \cal{L}h &= \inner{b}{\grad_x h} + \frac{1}{2} \operatorname{tr}\left(\sigma\sigma^T \grad_x \grad_x h\right); \quad h \in \Gamma(\R^d)
    \end{aligned}$
    &
    $\begin{aligned}
    \cal{S}g &= \inner{b}{g} + \frac{1}{2} \operatorname{tr}\left(\sigma\sigma^T \grad_x g\right); \quad g \in \Gamma(\R^d, \R^d)
    \end{aligned}$
    \\
    \hline
    RKHS &$\text{$\R$-valued kernel } k, \text{ RKHS } \cal{H} \subset \Gamma(\R^d)$
    & 
    $\text{$\R^d$-valued kernel } K, \text{ RKHS } \cal{H}^d \subset \Gamma(\R^d, \R^d)$
    \\
    \hline
    Definition&$\operatorname{KDS}(\cal{L}, \mu; \cal{H}_{\leq 1}) = \sup_{h \in \cal{H}_{\leq 1}} \E_{X \sim \mu}[\cal{L}h]$
    & 
    $\operatorname{SKDS}(\mathcal{S}, \mu; \cal{H}^d_{\leq 1}) = \sup_{g \in \cal{H}^d_{\leq 1}} \E_{X \sim \mu}[\mathcal{S} g]$
    \\
    \hline
    Closed form&
    $\E_{X \sim \mu}\E_{Y \sim \mu}\left[{\cal{L}_2}{\cal{L}_1}k(X,Y)\right]$&
    $\E_{X \sim \mu}\E_{Y \sim \mu}\left[\mathcal{S}_2{\mathcal{S}_1}K(X,Y)\right]$
    \\
    \hline
    Time (ms) &
    \footnotesize (RBF) $\!117\!\pm\!22$ / (Tilt.RBF) $\!183\!\pm\!27$ / (IMQ+) $\!293\!\pm\!37$
    &
    \footnotesize (RBF) $\boldsymbol{\!22\!\pm\!5}$ / (Tilt.RBF) $\boldsymbol{\!30\!\pm\!1}$ / (IMQ+) $\boldsymbol{\!46\!\pm\!3}$
    \\
    \hline
    \end{tabular}
    \caption{Structural and computational comparison of KDS and Stein-type KDS on $n = 1000$ random samples in $d = 20$ dimensions, averaged over $50$ repetitions.}
    \label{tab:kds_vs_skds}
\end{table*}

We note that the generator $\mathcal{L}$ used in the KDS involves an extra differentiation step compared to $\mathcal{S}$, the operator underlying the SKDS. Hence, $\mathcal{L}_1 \mathcal{L}_2 k(x,y)$ involves forth-order derivatives of $k$, whereas $\mathcal{S}_1 \mathcal{S}_2 k(x,y)$ only needs derivatives of the kernel up to order two. For a direct comparison of the two general explicit forms we refer to \cref{eq:simple closed form} for the SKDS and to \cref{app:KDS} for the KDS. This explains why, as we demonstrate next, the SKDS is computationally cheaper than the KDS. To benchmark computation time, we randomly generate a linear drift $b$ and linear diffusion $\sigma$, and mimic a gradient descent step by evaluating the KDS and SKDS with their respective kernels and computing derivatives with respect to $b$ and $\sigma$. We consider three kernels: an RBF kernel (RBF), a tilted RBF kernel (Tilt. RBF), and an IMQ+ kernel (IMQ+); cf.~\cref{app:Kernel Choices}. Results show that the SKDS achieves a speedup over the KDS by a factor greater than 5 (\cref{tab:kds_vs_skds}).

\section{Synthetic Experiments}\label{sec:Synthetic Experiments}
A key objective of causal modeling is to predict the effects of interventions, including those not observed during training. To evaluate this capability, we assess how well learned models generalize to unseen interventions by comparing their predicted interventional distributions to ground truth. Models are trained on interventional data with known targets and tested on interventions affecting previously unaltered variables.

\subsection{Setup}

We follow the synthetic evaluation framework of \citet{Lorch:2024}; see their Section~6, Appendix, and codebase~\citep{stadion2024} for details. The setup involves simulating dynamical systems, applying both training and test interventions, and evaluating predictive accuracy. We use the KDS-based causal model of \cite{Lorch:2024} and replace it with our SKDS variant. Our implementation is available online \citep{steinstadion2025}.

\paragraph{Data.}
We generate synthetic data from two system types: (i) sparse cyclic linear models (SCMs and stationary SDEs) and (ii) gene expression dynamics simulated with SERGIO \citep{dibaeinia2020sergio}, a stationary nonlinear stochastic process over sparse acyclic networks, grounded in the chemical Langevin equation with Hill-type nonlinearities. Causal structures $G \in \{0,1\}^{d \times d}$ are sampled from Erd\H{o}s–R\'enyi (fixed edge probability, expected degree $3$) and scale-free (preferential attachment, randomly directed) distributions. For each system, we create $50$ datasets with $10$ training and $10$ test interventions, each targeting a distinct variable and containing $1000$ samples. Interventions are additive shifts for both types of linear models and gene overexpression for SERGIO. During learning, we allow the more general shift–scale framework, modifying drift and diffusion $(\delta,\beta)$. All datasets are standardized by the observational mean and variance.

\begin{figure*}[t]
    \centering
    \includegraphics[width=1\linewidth]{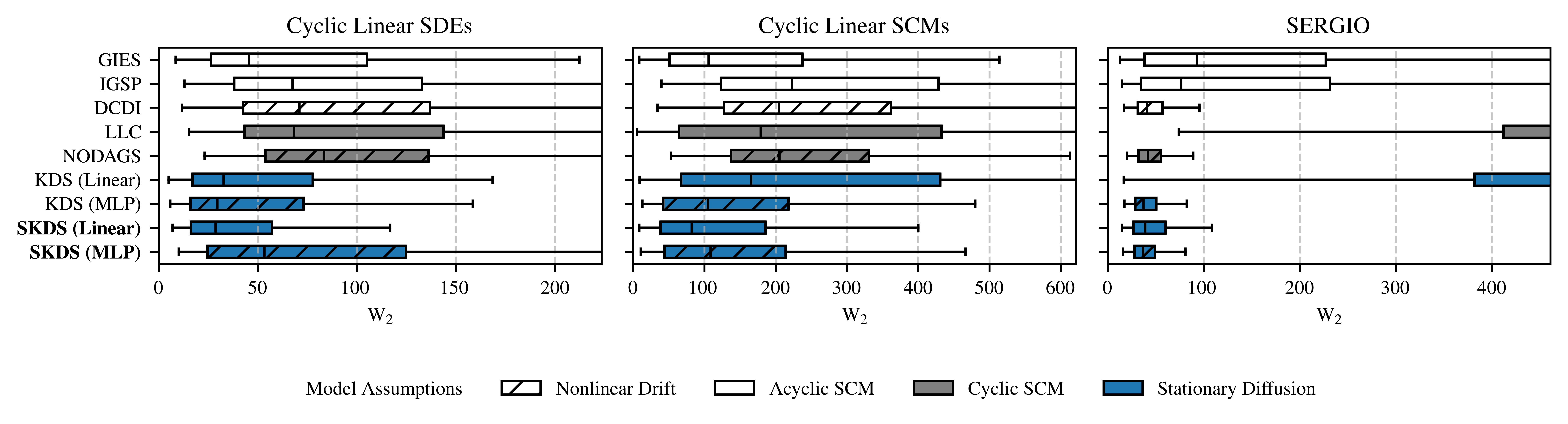}
    \caption{Benchmarking results for $d = 20$ variables with an Erd\H{o}s--R\'enyi causal structure. The Wasserstein-$2$ distance ($\mathcal{W}_2$) was computed from $10$ test interventions on unseen target variables across $50$ randomly generated systems. Box plots depict the medians and interquartile ranges (IQR), with whiskers extending to the largest value within $1.5$ times the IQR from the boxes.}
    \label{fig:Results-W2-ER}
\end{figure*}

\paragraph{Models and interventions.}
The Stein-type KDS plugs naturally into the framework of \citep[Algorithm 1]{Lorch:2024}. We sketch the learning algorithm in \cref{alg:SKDS learning} and give the python code lines in \cref{app:Code}.
The stationary diffusion processes are parameterized by either linear or multi-layer perceptron (MLP) mechanisms. The drift function \( b_\theta(x) \) is defined component-wise for each variable \( x_j \) as follows:
\begin{align*}
\text{Linear model:} \quad & f_{\theta_j}(x)_j = b_j + w_j \cdot x \\
\text{MLP model:} \quad & f_{\theta_j}(x)_j = b_j + w_j \cdot \gamma(U_j x + v_j) - x_j
\end{align*}
where \( \gamma \) denotes the sigmoid activation function. The diffusion matrix is parameterized as a diagonal matrix \( \sigma_\theta(x) = \operatorname{diag}(\exp(\sigma)) \), with learned log-standard deviations.
To remove the speed scaling invariance, we fix $w_j^j = -1$ in the linear and $u_j^j = 0$ in the MLP model (see \cref{subsec:Identifiability} and the note after \cref{prop:SKDS emp convex}).
Note that both model variants violate the boundedness assumptions on the drift and diffusion functions, introduced in \cref{lem:SKDS closed form} and subsequently used throughout all theoretical results. However, as discussed above, these assumptions can be informally relaxed in practice to merely require continuity.

Interventions $\phi$ are modeled as \emph{shift-scale interventions} on the known targets. For a variable \( x_j \), the intervention given by $\phi = (\delta, \beta)$ modifies the drift and diffusion as
\[
b_{\theta,\phi}(x)_j = f_\theta(x)_j + \delta, \qquad \sigma_{\theta,\phi}(x)_j = \beta \cdot \sigma_\theta(x)_j.
\]
\begin{algorithm}[ht]
\caption{Learning Causal Stationary Diffusions via the Stein-type KDS}
\label{alg:SKDS learning}
\normalsize 
\KwIn{Interventional datasets $\set{D_1, \ldots, D_M}$, kernel $k$, sparsity penalty $\lambda$, regularizer $R$}
\KwOut{Learned parameters $\theta$ and interventions $\set{\phi_1, \ldots, \phi_M}$}

Initialize model $\theta$ and interventions $\set{\phi_1, \ldots, \phi_M}$\;

\While{not converged}{
    Sample environment index $i \sim \text{Unif}([M])$\;
    Sample batch $D \sim D_i$\;
    Update $\theta$ and $\phi_i$ via gradient descent step
    $\propto - \nabla_{\theta, \phi_i} \left( \hat{\operatorname{SKDS}}(S^{\phi, \theta}, D; k) + \lambda R(\theta, \phi) \right)$\;
}
\Return{$\theta$, $\set{\phi_1, \ldots, \phi_M}$}
\end{algorithm}

We benchmark SKDS against six established baselines. The graph discovery methods GIES \citep{hauser2012characterization} and IGSP \citep{wang2017permutation} assume linear-Gaussian SCMs. They infer graph structure from interventional data, after which parameters are estimated in closed form. DCDI \citep{brouillard2020differentiable} extends this to nonlinear Gaussian SCMs parameterized by neural networks, jointly learning functional relations and noise variances. NODAGS \citep{sethuraman2023nodags} relaxes acyclicity, modeling nonlinear cyclic SCMs via residual normalizing flows. LLC \citep{hyttinen2012learning} also addresses cyclic structures but in a linear-Gaussian setting, with sparsity-regularized optimization. Interventions are simulated in all SCM-based approaches by biasing the distribution of the target variables. Finally, we include the KDS-based stationary diffusion model \citep{Lorch:2024}, which is based on the same SDE framework as SKDS and serves as the closest baseline. We refer to \citet{Lorch:2024} for further implementation and tuning details.

\paragraph{Metrics.}
We evaluate the learned causal models via test interventions that shift the mean of the target variables to match those observed in held-out interventional data. Performance is measured using the Wasserstein-$2$ distance (\( \mathcal{W}_2 \)) between samples generated by the models under these interventions and the corresponding true interventional samples. This metric allows comparison across methods with either explicit or implicit density representations. Additionally, we report the mean squared error (MSE) between the predicted and true empirical means of the interventional distributions; compare \citep[D.2]{Lorch:2024}.

\subsection{Results}

In the main text, we report results for the Erd\H{o}s--R\'enyi (ER) setting with Wasserstein distance as the evaluation metric; the full set of results and significance tests is provided in \cref{app:Results}. This focus is justified as the empirical behavior for ER and scale-free (SF) graphs is qualitatively similar. Moreover, in cyclic linear SDEs or SCMs all methods achieve comparably low MSE, while performance differences become apparent only once nonlinearity is introduced in the data. In the nonlinear case, MSE and Wasserstein distance lead to consistent conclusions.

The cyclic linear SDE setup aligns naturally with our stationary SDE formulation, and both SKDS variants perform competitively in this setting. In \cref{fig:Results-W2-ER}, we observe that the acyclic method GIES remains strong despite model misspecification, whereas cyclic approaches such as LLC and NODAGS underperform. SKDS (MLP) is outperformed by the other stationary diffusion models (see \cref{tab:SKDS significance tests ER-Wasserstein}), which may reflect suboptimal hyperparameters or the inherent tradeoff between model complexity and optimization.

\begin{table}[t]
\centering
\setlength{\tabcolsep}{3pt} 
\renewcommand{\arraystretch}{1.1} 
\begin{tabular}{lcccccc}
& \multicolumn{2}{c}{SDE-ER} & \multicolumn{2}{c}{SCM-ER} & \multicolumn{2}{c}{SERGIO-ER} \\
Baseline & Lin & MLP & Lin & MLP & Lin & MLP \\
\hline
GIES 
& $\ast$ & 
&  $\ast$ & 
& $\ast$ & $\ast$ \\
IGSP 
& $\ast$ & $\ast$ 
& $\ast$ & $\ast$ 
& $\ast$ & $\ast$ \\
DCDI 
& $\ast$ & $\ast$ 
& $\ast$ & $\ast$ 
& & $\ast$ \\
LLC 
& $\ast$ & $\ast$ 
& $\ast$ & $\ast$ 
& $\ast$ & $\ast$ \\
NODAGS 
& $\ast$ & $\ast$ 
& $\ast$ & $\ast$ 
& $\ast$ & $\ast$ \\
KDS (Linear) 
&  & \textcolor{red}{$\ast$} 
& $\ast$ & $\ast$ 
& $\ast$ & $\ast$ \\
KDS (MLP) 
&  & \textcolor{red}{$\ast$} 
& $\ast$ &
& & \\
\end{tabular}
\caption{Significance tests for \textbf{SKDS (Linear)} (Lin) and \textbf{SKDS (MLP)} (MLP) on ER graphs with Wasserstein-$2$ distance. A black star indicates a statistically significant improvement of our method (margin $\geq 5\%$) over the considered baseline method, while a red star indicates the baseline significantly outperforming ours.}
\label{tab:SKDS significance tests ER-Wasserstein}
\end{table}

The results for the cyclic linear SCMs are largely similar to those for the cyclic linear SDEs, with stationary diffusion models and GIES performing best. This is particularly notable because the generated data satisfies the assumptions of LLC and NODAGS, suggesting that while these methods achieve reasonable MSE (see \cref{fig:Results-W2-ER-Full} in \cref{app:Results}) -- possibly capturing mean behavior of interventional distributions -- they struggle to accurately model full distributional properties, as reflected in the Wasserstein distance.

On the SERGIO gene expression data, which introduces realistic nonlinearities and model mismatch, methods with nonlinear drift dominate. Here, SKDS (MLP) is among the top performers, while SKDS (Linear) remains competitive despite its restricted functional form.
\section{Conclusion}
We introduced the Stein-type KDS as a principled framework for learning stationary diffusions. Theoretically, SKDS vanishes precisely for the class of stationary reversible diffusions whose stationary distribution matches the target. Furthermore, for  linear parametrization (including expansions in nonlinear basis functions), the SKDS objective is convex with a quasiconvex empirical estimator. Empirically, SKDS matches or improves upon KDS and other competitive baselines across linear and nonlinear, cyclic and acyclic settings, while achieving a substantial reduction in computation time (\Cref{tab:kds_vs_skds}).
In summary, our findings highlight the robustness and flexibility of stationary diffusions as a computationally efficient foundation for causal modeling, providing a principled alternative to existing approaches.


\subsubsection*{Acknowledgements}
The authors acknowledge support from the European Research Council (ERC) under the European Union’s Horizon 2020 research and innovation programme (grant agreement No 883818).
Fabian Bleile acknowledges institutional support from the TUM Georg Nemetschek
Institute for Artificial Intelligence for the Built World and the Munich Center for Machine Learning.
Sarah Lumpp was further supported by the Munich Data Science Institute as well as the DAAD programme Konrad Zuse Schools of Excellence in Artificial Intelligence, sponsored by the Federal Ministry of Research, Technology and Space.

\bibliography{literature}

\begin{thebibliography}{34}
\providecommand{\natexlab}[1]{#1}
\providecommand{\url}[1]{\texttt{#1}}
\expandafter\ifx\csname urlstyle\endcsname\relax
  \providecommand{\doi}[1]{doi: #1}\else
  \providecommand{\doi}{doi: \begingroup \urlstyle{rm}\Url}\fi

\bibitem[Anastasiou et~al.(2023)Anastasiou, Barp, Briol, Ebner, Gaunt,
  Ghaderinezhad, Gorham, Gretton, Ley, Liu, et~al.]{anastasiou2023stein}
Andreas Anastasiou, Alessandro Barp, Fran{\c{c}}ois-Xavier Briol, Bruno Ebner,
  Robert~E Gaunt, Fatemeh Ghaderinezhad, Jackson Gorham, Arthur Gretton,
  Christophe Ley, Qiang Liu, et~al.
\newblock {S}tein's method meets computational statistics: {A} review of some
  recent developments.
\newblock \emph{Statistical Science. A Review Journal of the Institute of
  Mathematical Statistics}, 38\penalty0 (1):\penalty0 120--139, 2023.

\bibitem[Arnold(1974)]{arnold1974stochastic}
Ludwig Arnold.
\newblock \emph{Stochastic differential equations: theory and applications}.
\newblock John Wiley \& Sons, New York-London-Sydney, 1974.

\bibitem[Barp et~al.(2019)Barp, Briol, Duncan, Girolami, and
  Mackey]{barp2019minimum}
Alessandro Barp, Fran{\c{c}}ois{-}Xavier Briol, Andrew~B. Duncan, Mark~A.
  Girolami, and Lester~W. Mackey.
\newblock Minimum stein discrepancy estimators.
\newblock In \emph{Advances in Neural Information Processing Systems},
  volume~32, pages 12964--12976, 2019.

\bibitem[Bleile(2025)]{steinstadion2025}
Fabian Bleile.
\newblock Steinstadion: A framework for learning on diffusions.
\newblock \url{https://github.com/fbleile/steinstadion/tree/SKDS}, 2025.

\bibitem[Brouillard et~al.(2020)Brouillard, Lachapelle, Lacoste,
  Lacoste-Julien, and Drouin]{brouillard2020differentiable}
Philippe Brouillard, S{\'e}bastien Lachapelle, Alexandre Lacoste, Simon
  Lacoste-Julien, and Alexandre Drouin.
\newblock Differentiable causal discovery from interventional data.
\newblock In \emph{Advances in Neural Information Processing Systems},
  volume~33, pages 21865--21877, 2020.

\bibitem[Carmeli et~al.(2006)Carmeli, De~Vito, and Toigo]{carmeli2006vector}
Claudio Carmeli, Ernesto De~Vito, and Alessandro Toigo.
\newblock Vector valued reproducing kernel {H}ilbert spaces of integrable
  functions and {M}ercer theorem.
\newblock \emph{Analysis and Applications}, 4\penalty0 (4):\penalty0 377--408,
  2006.

\bibitem[Cohn(2013)]{wiki:BochnerIntegral}
Donald~L. Cohn.
\newblock \emph{Measure theory}.
\newblock Birkh\"auser Advanced Texts: Basler Lehrb\"ucher.
  Birkh\"auser/Springer, New York, 2 edition, 2013.

\bibitem[Dettling et~al.(2023)Dettling, Homs, Am{\'e}ndola, Drton, and
  Hansen]{dettling2023identifiability}
Philipp Dettling, Roser Homs, Carlos Am{\'e}ndola, Mathias Drton, and
  Niels~Richard Hansen.
\newblock {I}dentifiability in continuous {L}yapunov models.
\newblock \emph{SIAM Journal on Matrix Analysis and Applications}, 44\penalty0
  (4):\penalty0 1799--1821, 2023.

\bibitem[Dibaeinia and Sinha(2020)]{dibaeinia2020sergio}
Payam Dibaeinia and Saurabh Sinha.
\newblock Sergio: A single-cell expression simulator guided by gene regulatory
  networks.
\newblock \emph{Cell Systems}, 11\penalty0 (3):\penalty0 252--271.e11, 2020.

\bibitem[Ethier and Kurtz(1986)]{ethier2009markov}
Stewart~N. Ethier and Thomas~G. Kurtz.
\newblock \emph{Markov processes}.
\newblock Wiley Series in Probability and Mathematical Statistics: Probability
  and Mathematical Statistics. John Wiley \& Sons, Inc., New York, 1986.

\bibitem[Gorham and Mackey(2015)]{gorham2015measuring}
Jackson Gorham and Lester Mackey.
\newblock {M}easuring sample quality with {S}tein's method.
\newblock In \emph{Advances in Neural Information Processing Systems},
  volume~28, pages 226--234, 2015.

\bibitem[Gorham and Mackey(2017)]{gorham2017measuring}
Jackson Gorham and Lester Mackey.
\newblock {M}easuring sample quality with kernels.
\newblock In \emph{Proceedings of the 34th International Conference on Machine
  Learning}, volume~70, pages 1292--1301, 2017.

\bibitem[Gorham et~al.(2019)Gorham, Duncan, Vollmer, and
  Mackey]{gorham2019measuring}
Jackson Gorham, Andrew~B Duncan, Sebastian~J. Vollmer, and Lester Mackey.
\newblock {M}easuring sample quality with diffusions.
\newblock \emph{The Annals of Applied Probability}, 29\penalty0 (5):\penalty0
  2884--2928, 2019.

\bibitem[Hauser and B\"uhlmann(2012)]{hauser2012characterization}
Alain Hauser and Peter B\"uhlmann.
\newblock Characterization and greedy learning of interventional {M}arkov
  equivalence classes of directed acyclic graphs.
\newblock \emph{Journal of Machine Learning Research (JMLR)}, 13:\penalty0
  2409--2464, 2012.

\bibitem[Huang et~al.(2015)Huang, Ji, Liu, and Yi]{huang2015steady}
Wen Huang, Min Ji, Zhenxin Liu, and Yingfei Yi.
\newblock Steady states of {F}okker-{P}lanck equations: {I}. {E}xistence.
\newblock \emph{Journal of Dynamics and Differential Equations}, 27\penalty0
  (3-4):\penalty0 721--742, 2015.

\bibitem[Hyttinen et~al.(2012)Hyttinen, Eberhardt, and
  Hoyer]{hyttinen2012learning}
Antti Hyttinen, Frederick Eberhardt, and Patrik~O. Hoyer.
\newblock Learning linear cyclic causal models with latent variables.
\newblock \emph{The Journal of Machine Learning Research}, 13\penalty0
  (1):\penalty0 3387--3439, 2012.

\bibitem[Jacobsen(1993)]{jacobsen1993brief}
Martin Jacobsen.
\newblock {A} brief account of the theory of homogeneous {G}aussian diffusions
  in finite dimensions.
\newblock \emph{Frontiers in Pure and Applied Probability}, 1:\penalty0 86--94,
  1993.

\bibitem[Kanagawa et~al.(2022)Kanagawa, Barp, Gretton, and
  Mackey]{kanagawa2022controlling}
Heishiro Kanagawa, Alessandro Barp, Arthur Gretton, and Lester Mackey.
\newblock Controlling moments with kernel stein discrepancies.
\newblock \emph{arXiv e-prints}, pages arXiv--2211, 2022.

\bibitem[Liu et~al.(2016)Liu, Lee, and Jordan]{liu2016kernelized}
Qiang Liu, Jason Lee, and Michael Jordan.
\newblock {A} kernelized {S}tein discrepancy for goodness-of-fit tests.
\newblock In \emph{Proceedings of the 33rd International Conference on Machine
  Learning}, pages 276--284, 2016.

\bibitem[Lorch(2024)]{stadion2024}
Lars Lorch.
\newblock Stadion: A framework for learning on diffusions.
\newblock \url{https://github.com/larslorch/stadion}, 2024.
\newblock Accessed: 2024-11-09.

\bibitem[Lorch et~al.(2024)Lorch, Krause, and Sch{\"o}lkopf]{Lorch:2024}
Lars Lorch, Andreas Krause, and Bernhard Sch{\"o}lkopf.
\newblock {C}ausal {M}odeling with {S}tationary {D}iffusions.
\newblock In \emph{Proceedings of The 27th International Conference on
  Artificial Intelligence and Statistics}, volume 238, pages 1927--1935. PMLR,
  2024.

\bibitem[\O~ksendal(2003)]{oksendal2013stochastic}
Bernt \O~ksendal.
\newblock \emph{Stochastic differential equations}.
\newblock Universitext. Springer-Verlag, Berlin, 6 edition, 2003.

\bibitem[Pavliotis(2014)]{pavliotis2014stochastic}
Grigorios~A. Pavliotis.
\newblock \emph{Stochastic processes and applications}, volume~60 of
  \emph{Texts in Applied Mathematics}.
\newblock Springer, New York, 2014.

\bibitem[Pearl(2009)]{Pearl:2009}
Judea Pearl.
\newblock \emph{Causality}.
\newblock Cambridge University Press, Cambridge, 2 edition, 2009.

\bibitem[Pigola and Setti(2014)]{pigola2014global}
Stefano Pigola and Alberto~G. Setti.
\newblock \emph{Global divergence theorems in nonlinear {PDE}s and geometry},
  volume~26 of \emph{Ensaios Matem\'aticos}.
\newblock Sociedade Brasileira de Matem\'atica, Rio de Janeiro, 2014.

\bibitem[Sethuraman et~al.(2023)Sethuraman, Lopez, Mohan, Fekri, Biancalani,
  and Huetter]{sethuraman2023nodags}
Muralikrishnna~G. Sethuraman, Romain Lopez, Rahul Mohan, Faramarz Fekri,
  Tommaso Biancalani, and Jan-Christian Huetter.
\newblock Nodags-flow: Nonlinear cyclic causal structure learning.
\newblock In \emph{Proceedings of The 26th International Conference on
  Artificial Intelligence and Statistics}, volume 206, pages 6371--6387. PMLR,
  2023.

\bibitem[Sokol and Hansen(2014)]{hansen2014causal}
Alexander Sokol and Niels~Richard Hansen.
\newblock Causal interpretation of stochastic differential equations.
\newblock \emph{Electronic Journal of Probability}, 19:\penalty0 no. 100, 24,
  2014.

\bibitem[Spirtes et~al.(2000)Spirtes, Glymour, and Scheines]{Spirtes:2000}
Peter Spirtes, Clark Glymour, and Richard Scheines.
\newblock \emph{Causation, prediction, and search}.
\newblock Adaptive Computation and Machine Learning. MIT Press, Cambridge, MA,
  2 edition, 2000.

\bibitem[Stein(1972)]{stein1972bound}
Charles Stein.
\newblock {A} bound for the error in the normal approximation to the
  distribution of a sum of dependent random variables.
\newblock \emph{Proceedings of the sixth Berkeley symposium on mathematical
  statistics and probability, volume 2: Probability theory}, 6:\penalty0
  583--603, 1972.

\bibitem[Steinwart and Christmann(2008)]{christmann2008support}
Ingo Steinwart and Andreas Christmann.
\newblock \emph{Support vector machines}.
\newblock Information Science and Statistics. Springer, New York, 2008.

\bibitem[Tropp(2015)]{tropp2015introduction}
Joel~A. Tropp.
\newblock \emph{An introduction to matrix concentration inequalities},
  volume~8.
\newblock Now Publishers, Inc., 2015.

\bibitem[Varando and Hansen(2020)]{varando2020graphical}
Gherardo Varando and Niels~Richard Hansen.
\newblock Graphical continuous {L}yapunov models.
\newblock In \emph{Proceedings of the Thirty-Sixth Conference on Uncertainty in
  Artificial Intelligence, {UAI} 2020}, volume 124, pages 989--998. {AUAI}
  Press, 2020.

\bibitem[Wang et~al.(2023)Wang, Geng, Huang, Huang, and
  Gong]{wang2024generator}
Yuanyuan Wang, Xi~Geng, Wei Huang, Biwei Huang, and Mingming Gong.
\newblock {G}enerator identification for linear {SDE}s with additive and
  multiplicative noise.
\newblock In \emph{Advances in Neural Information Processing Systems},
  volume~36, pages 64103--64138, 2023.

\bibitem[Wang et~al.(2017)Wang, Solus, Yang, and Uhler]{wang2017permutation}
Yuhao Wang, Liam Solus, Karren~D. Yang, and Caroline Uhler.
\newblock Permutation-based causal inference algorithms with interventions.
\newblock In \emph{Advances in Neural Information Processing Systems},
  volume~30, pages 5822--5831, 2017.

\end{thebibliography}






%

%

\onecolumn
\thispagestyle{empty}
\appendixtitle{
Supplementary Materials}
\appendix

\section{ADDITIONAL BACKGROUND}

\subsection{Vector Calculus}\label{app:Vector Calculus}
We give notation and important identities from vector calculus (based upon \citep[A.1]{barp2019minimum}). For $f \in \Gamma(\R^n, \R^n)$, we denote $\inner{\grad_x}{f(x)}$ as the divergence of $f$,  $\operatorname{div}f(x) = \sum_{i \in [n]} \partial_i f_i(x)$. For a function \( g \in \Gamma(\mathcal{X}) \), \( v \in \Gamma(\mathcal{X}, \mathbb{R}^d) \) and \( A \in \Gamma(\mathcal{X}, \mathbb{R}^{d \times d}) \) with components \( A_{ij} \), \( v_i \), \( g \), we have $(\nabla g)_i = \partial_i g$, $(\inner{\grad}{A})_i = \inner{\grad}{A_i}$, where $A_i$ is the $i$-th column of $A$. We define $A^j$ as the $j$-th row of $A$ respectively. 
Moreover, we have \( (\nabla v)_{ij} = \partial_j v_i \), \( \nabla^2 f \equiv \nabla (\nabla f) \). We have the following identities,
\begin{align*}
\inner{\nabla}{gv} &= \partial_i (g v_i) = v_i \partial_i g + g \partial_i v_i = \inner{\grad g}{v} + g \inner{\grad}{v},\\
\inner{\grad}{gA} &= \partial_i (g A_{ij}) e_j = (A_{ij} \partial_i g + g \partial_i A_{ij}) e_j = \inner{\nabla g}{A} + g \inner{\grad}{A},\\
\inner{\nabla}{Av} &= \partial_i (A_{ij} v_j) = \inner{\inner{\grad}{A}}{v} + \operatorname{Tr} (A \nabla v),
\end{align*}
where we treat $\inner{\nabla}{A}$ and \( \nabla g \) as column vectors. For a matrix $M \in \mathbb{R}^{m \times n}$, the column-wise vectorization of $M$ is obtained by stacking its columns on top of each other, i.e., \[
\operatorname{vec}(M) = 
\begin{bmatrix}
M_{1} \\
M_{2} \\
\vdots \\
M_{n}
\end{bmatrix} \in \mathbb{R}^{mn}.
\]

\subsection{Kernel Theory}\label{app:Kernel Theory}
This section reviews key concepts from kernel theory  \citep{christmann2008support,carmeli2006vector}.

\begin{definition}[Positive Semidefinite Kernel]\label{def:PSD Kernel}
    Let $X$ be an arbitrary set. A map $k: X \times X \ra \R$ is called \emph{positive semidefinite kernel} (PSD kernel) if and only if for all $n \in \N$ and all $x \in X^n$ the $n \times n$ matrix $G$ with entries $G_{ij} := k(x_i, x_j)$ is positive semidefinite, i.e., \( c^T G c \geq 0 \) for all \( c \in \mathbb{R}^n \).
\end{definition}

\begin{theorem}[PSD Kernels and Feature Maps]\label{th:PSD Kernels and Feature Maps}
    Let $X$ be any set, and let $k:X \times X \ra \R$.
    \begin{enumerate}
        \item[(i)] Then $k$ is a PSD kernel, if there is an inner product space $\cal{H}$ and a map $\phi: X \ra \cal{H}$ such that
        \begin{equation}\label{eq:kernel x feature map}
            k(x,y) = \inner{\phi(x)}{\phi(x)} \quad \text{for all }x,y\in X.
        \end{equation}
        \item[(ii)] Conversely, if $k$ is a PSD kernel, then there exists a Hilbert space $\cal{H}$ and a map $\phi : X \ra \cal{H}$ such that \cref{eq:kernel x feature map} holds.
    \end{enumerate}
\end{theorem}

For a given PSD kernel, the corresponding \emph{feature map} $\phi$ and \emph{feature space} $\cal{H}$ are not unique. However, there is a canonical choice for the feature space, a so-called reproducing kernel Hilbert space.

\begin{definition}[Reproducing Kernel Hilbert Space]
Let \( X \) be a set, and let \( \mathcal{H} \subseteq \R^X \) be an $\R$-Hilbert space\footnote{i.e., inner-product space which is a complete metric space with its norm induced by the inner product.} of functions on \( X \) with addition \( (f + g)(x) := f(x) + g(x) \) and multiplication \( (\lambda f)(x) := \lambda f(x) \). \( \mathcal{H} \) is called a \emph{reproducing kernel Hilbert space} (RKHS) on \( X \) if for all \( x \in X \) the linear functional \( \delta_x : \mathcal{H} \rightarrow \R \), \( \delta_x(f) := f(x) \) is bounded (i.e., \( \sup_{f \in \mathcal{H} \setminus \{0\}} \frac{|f(x)|}{\|f\|} < \infty \)).
\end{definition}

Note that since $\delta_x$ is linear, boundedness is equivalent to continuity. That is, the defining property of a RKHS is that evaluation of its functions at arbitrary points is continuous with respect to varying the function. Moreover, the RKHS $\cal{H}$ to a kernel $k$ satisfies the \emph{reproducing property} that $f(x) = \inner{f}{k(\cdot, x)}_{\cal{H}}$ for all $x \in \R^d$.

\begin{definition}[Universal Kernel]\label{def:universal kernel}
A PSD kernel $k: X \times X \to \mathbb{R}$ on a metric space $X$ is called \emph{universal} if
\[
E := \operatorname{span}\{ k_x : X \to \mathbb{R} \mid k_x(y) = k(x,y), \, x \in X \}
\]
is dense in the set of continuous functions with compact support, $C_c(X, \mathbb{R})$, i.e., for any $f \in C_c(X, \mathbb{R})$ and $\varepsilon > 0$, there exists $g \in E$ such that $\|g - f\|_\infty < \varepsilon$.

Note that if $\phi: X \to \mathcal{H}$ is a feature map corresponding to $k$, then for every $g \in E$ there exists $w \in \mathcal{H}$ such that $g = \langle w, \phi(\cdot) \rangle$.
\end{definition}

This notion extends naturally to the vector-valued case. 
If $K(x,y) = k(x,y) I_d$ with $k$ universal, then the associated RKHS is 
$\mathcal{H}^d := \mathcal{H} \times \cdots \times \mathcal{H}$, 
and finite linear combinations of kernel sections $K_x(\cdot)$ are dense in 
$C_c(X, \mathbb{R}^d)$. Hence $K$ is universal in the vector-valued sense.

\begin{definition}[Differentiability of a Kernel]\label{def:Differentiability of a Kernel}
    Let $k$ be a kernel on $\R^n$ with RKHS $\cal{H}$. We adopt the notation from \citep{christmann2008support} and regarding differentiability we interpret $k$ as a function $\Tilde k : \R^{2n} \ra \R$. Then define $\partial_{i}\partial_{i+n} k = \partial_{i}\partial_{i+n} \Tilde k$. More generally, define $\partial^{\alpha, \alpha} = \partial_1^{\alpha_1} \ldots \partial_n^{\alpha_n}\partial_{n+1}^{\alpha_{n+1}} \ldots \partial_{2n}^{\alpha_{2n}}$, where $\alpha \in \N^n$. We say $k$ is $m$-times continuously differentiable, and write $k \in C^{(m, m)}$, if $\partial^{\alpha, \alpha} k$ exists and is continuous for all $\sum_{i \in [n]} \alpha_i \leq m$. 
\end{definition}

The theory of kernels can be extended to the case where $k$ maps to a Hilbert space $\cal{K}$ \citep{carmeli2006vector}. The resulting RKHS consists of functions from $X$ to $\cal{K}$. In particular, we are interested in the setting with $\R^d$-valued functions.

\begin{definition}[RKHS for $\R^d$-valued functions]\citep[Def.~2.2]{carmeli2006vector}
    An $\R^d$-valued kernel of positive type on $X \times X$ is a map $K: X \times X \ra \R^{d \times d}$ such that for all $N \in \N$, $x_1, \ldots, x_N \in X$ and $c_1, \ldots, c_N \in \R$, we have
    \begin{equation}
        \sum_{i,j \in [N]} c_i c_j \inner{K(x_j, x_i)v}{v}_{\R^d} \geq 0 \quad \text{for all }v\in \R^d,
    \end{equation}
    or equivalently, $K(x,y)$ is positive semi-definite for any $x$, $y \in X$. 
\end{definition}

Consider the pre-Hilbert space
    \begin{equation}
        \begin{aligned}
            \cal{H}_0^d &= \operatorname{span} \set{K(\cdot, x)v \mid x \in X, v \in \R^d},
        \end{aligned}
    \end{equation}
    with $f = \sum_i^n a_i K(\cdot, x_i)v_i \in \cal{H}_0^d$ for $a_i \in \R$, $x_i \in X$, and $v_i \in \R^d$ for all $i \in [n]$. A sesquilinear form on $\cal{H}_0^d \times \cal{H}_0^d$ is defined by 
    \begin{equation}
        \begin{aligned}
            \inner{f}{g}_{\cal{H}_0^d} &= \sum_{i}^n \sum_j^m a_i b_j \inner{ K(y_j, x_i)v_i}{w_j}_{\R^d};
        \end{aligned}
    \end{equation}
  compare \citep[Proposition 2.3]{carmeli2006vector}. Then, the unique RKHS $\cal{H}^d$ with reproducing kernel $K$ is defined as the completion of $\cal{H}_0^d$ with respect to $\inner{\cdot}{\cdot}_{\cal{H}_0^d}$. This means
    \begin{equation}
        \begin{aligned}
            \cal{H}^d = \Big\{f = \sum_i a_i K(\cdot, x_i)v_i &\mid a_i \in \R \text{, } x_i \in X \text{, and } v_i \in \R^d \text{ for all }i \in \N \text{ such that}\\
            &\norm{f}^2_{\cal{H}^d} := \lim_{n \ra \infty} \norm{\sum_i^n a_i K(\cdot, x_i)v_i}^2_{\cal{H}_0^d} = \sum_{i,j} a_i a_j \inner{ K(x_j, x_i)v_i}{v_j}_{\R^d} <\infty \Big\}.
        \end{aligned}
    \end{equation}

The reproducing property in an RKHS $\cal{H}^d$ defined by an $\R^d$-valued kernel $K$ reads as follows. For $f \in \cal{H}^d$ and $x \in X$,
\begin{equation}\label{eq:R^d valued kernel reproducing property}
    \inner{f(x)}{v}_{\R^d} = \inner{f}{K(\cdot, x)v}_{\cal{H}^d}, \quad \text{for all }v\in \R^d.
\end{equation}
We note that any $\R$-valued kernel $k$ of positive type on $X \times X$ defines an $\R^d$-valued kernel of positive type on $X \times X$ together with a symmetric positive definite matrix $B$ by $K = k B$. We denote the RKHS of $k$ with $\cal{H}$. Specifically, we will be interested in the case when $B = I_d$. Then, for $f \in \cal{H}^d$, it holds that $f_i \in \cal{H}$ as
\begin{align*}
    \norm{f}^2_{\cal{H}^d} &= \sum_{i,j} a_i a_j \inner{ K(x_j, x_i)v_i}{v_j}_{\R^d}\\
    &= \sum_{i,j} a_i a_j k(x_j, x_i) \sum_{k \in [d]}(v_i)_k (v_j)_k\\
    &= \sum_{k \in [d]} \sum_{i,j} a_i a_j k(x_j, x_i) (v_i)_k (v_j)_k\\
    &= \sum_{k \in [d]} \norm{f_k}^2_{\cal{H}}.
\end{align*}

\subsection{Stochastic Differential Equations}

Next, we provide a short introduction to stochastic differential equations; for further details we refer to \citep{oksendal2013stochastic}.

Differential equations (DEs) are a natural choice to model deterministic processes and put (physical) quantities and their rates of change in relation. This approach may fail to accurately model a process if its randomness effects quantities too much. Discovered by the botanist Robert Brown, Brownian motion is the irregular thermal movement of small particles in liquids and gases. The physical quantities of such particles like speed and position were observed to be ``too random'' to be accurately modeled with DEs. A model that adjusts for randomness in a differential equation is called a Stochastic Differential Equation (SDE), and a common class of SDEs--Itô diffusions--extends ordinary differential equations by explicitly incorporating stochastic terms.

\begin{definition}[Itô diffusion]
    An Itô diffusion is a stochastic process $(X_t)_t$ that can be written in the form of \begin{equation}\label{ito diffusion}
        dX_t = b(X_t, t)dt + \sigma(X_t, t)dB_t, \quad t \in [0,T],\, X_0 = Z,
    \end{equation}
    with $(B_t)_{t \geq 0}$ Brownian motion, $T > 0$, $b,\sigma :\R^n \times [0,T] \ra \R^n$ measurable, and $Z \in L^2(\Omega)$ a square integrable random variable independent of the driving noise process.
    Define an autonomous (or time-homogeneous) Itô diffusion through
    \begin{equation}
        \label{auto ito process}
        dX_t = b(X_t)dt + \sigma(X_t)dB_t,
    \end{equation}
    where the drift $b$ and the diffusion (or noise) $\sigma$ only depend on $X_t$ and not on past values $X_s$ for $s < t$. This ensures a strong solution $(X_t)_{t \geq 0}$ of \cref{ito diffusion}, if it exists, to be Markov. 
\end{definition}
If $\sigma(x) = \sigma$ then the noise is called additive, otherwise multiplicative. In the modeling of physical systems using SDEs, additive noise is usually due to thermal fluctuations, whereas multiplicative noise is due to noise in some control parameter.

\begin{example}[Ornstein-Uhlenbeck Process]
\label{ex:ornstein-uhlenbeck process}
    An Ornstein-Uhlenbeck process is defined by the SDE
    \begin{equation}
        dX_t = M (X_t - a) dt + D dB_t
    \end{equation}
    with invertible drift matrix $M \in \R^{n \times n}$ and diffusion matrix $D \in \R^{n \times n}$. If the matrix $M$ is \emph{stable}, i.e., its eigenvalues have negative real parts, the affine linear drift introduces a restoring force towards the mean $a \in \R^n$.  For stable $M$ and positive definite $DD^T$, the Ornstein-Uhlenbeck process has a unique stationary solution $\cal{N}(a, \Sigma)$, where  $\Sigma$ solves $M \Sigma + \Sigma M^T = - DD^T$ \citep[Theorem (8.2.12)]{arnold1974stochastic}.
\end{example}

\subsection{Generator and the Martingale Problem}\label{app:SDE-Generator-Martingale Problem}
The generator $\cal{A}$ of a stochastic process defined by \cref{eq:SDE} = \cref{auto ito process} is a linear operator that acts on a function $f\in C_b(\R^n)$ describing the infinitesimal expected change when it is applied to the stochastic process.
For $f \in C^2(\R^n)$ and sufficiently regular $b$ and $\sigma$, the generator takes the form of a differential operator,
\begin{equation}\label{differential operator}
\begin{aligned}
    \cal{L} (f) (x) &= \inner{b(x)}{\grad f(x)} + \frac{1}{2} \operatorname{tr}\left(a(x)\nabla \nabla f(x)\right),
\end{aligned}
\end{equation}
where $a(x) = \sigma(x)\sigma(x)^T$ is called the \emph{covariance coefficient function}.
The martingale problem (see \citep{ethier2009markov}) connects the generator $\cal{A}$ of an SDE \cref{eq:SDE} with its stationary solution $P$. In specific, by the arguments in \citep[chapter~4]{ethier2009markov}, if \cref{eq:SDE} has a unique strong solution then $\mu$ is a stationary solution if and only if
\begin{equation}\label{eq:martingale problem}
    \E_{X \sim \mu}[\cal{A}f(X)] = 0 \text{ for all }f \in \cal{D}.
\end{equation}
Importantly, the domain $\cal{D}$ of $\cal{A}$ can be replaced by a smaller subset, called the \emph{core} of $\cal{A}$, without affecting the result. For sufficiently regular drift $b$ and diffusion $\sigma$ such a core can be shown to be the space of smooth functions with compact support $C^\infty_c$ (see \citep[Appendix B.3]{Lorch:2024}). As $C^\infty_c \subseteq C^2$, it follows that $\cal{A}$ takes the form of $\cal{L}$ in \cref{eq:martingale problem}, which further simplifies the characterization.

The family of SDEs that admit a stationary solution with density $\mu$ are referred to as \emph{$\mu$-targeted diffusions}. Notably, the class of $\mu$-targeted diffusions can be characterized: for continuously differentiable drift and diffusion functions, $\mu$ is a stationary density of a process solving \cref{eq:SDE} if and only if $b$ takes the form $b(x) = \frac{1}{2 \mu(x)} \inner{\grad}{\mu(x)a(x)} + h(x)$ for a non-reversible component $h \in C^1$ satisfying $\inner{\grad}{\mu(x)h(x)} = 0$ for all $x \in \R^d$. Moreover, if $h$ is $\mu$-integrable, then there exists a differentiable, $\mu$-integrable, and skew-symmetric $d\times d$ matrix-valued function $c$, called \emph{stream coefficient}, such that the differential operator $\cal{L}$ takes the form
\begin{equation}\label{eq:generator density version}
    (\cal{L} f)(x) = \frac{1}{2 \mu(x)} \inner{\grad}{\mu(x)(a(x) + c(x)) \grad f(x)}
\end{equation}
on $f \in C^2 \cap \operatorname{dom}(\cal{L})$.
The drift then has the form 
\begin{equation}\label{eq:drift density version}
    \begin{aligned}
        b(x) &= \frac{1}{2 \mu(x)} \inner{\grad}{\mu(x)(a(x) + c(x))} \\
        &= \frac{1}{2 \mu(x)} \inner{\grad}{\mu(x)a(x)} + \underbrace
{\frac{1}{2 \mu(x)} \inner{\grad}{\mu(x)c(x)}}_{= h}.
\end{aligned}
\end{equation}
The non-reversible $h$ can be related to the stationary probability flux $J = \mu h = \frac{1}{2} \inner{\grad}{\mu(x)c(x)}$, which -- in stationarity -- is the divergence-free part of the drift. The condition that the stationary probability flux vanishes, i.e.,
\begin{equation}\label{eq:detailed balance condition}
    J = 0
\end{equation}
is called \emph{detailed balance condition}. The detailed balance condition holds if and only if the diffusion is reversible \citep[Prop. 4.5.]{pavliotis2014stochastic}.

For later use, we can write $h$ also in the form
\begin{equation}\label{eq:non-reversible drift}
    h = \frac{1}{2}\left(c \grad \log \mu + \inner{\grad}{c} \right).
\end{equation}

We refer to \Citet{gorham2019measuring, pavliotis2014stochastic} for the details.

\subsection{Stein Discrepancy}\label{app:Stein Discrepancy}
Comparing probability distributions is essential for tasks like model evaluation, data generation, and hypothesis testing. A natural starting point is the maximum deviation between their expectations over a class of real-valued test functions $\cal{F}$ \citep{gorham2015measuring}, given by
\begin{equation}\label{eq:max expected deviation}
    d_\cal{F}(Q,P) := \sup_{f \in \cal{F}} \abs{\mathbb{E}_{X \sim Q}[f(X)] - \mathbb{E}_{X \sim P}[f(X)]}.
\end{equation}

Stein discrepancies are a family of statistical divergences, defined as maximal deviations arising from Stein's method~\citep{stein1972bound}, that are particularly useful due to their computability. The following definition makes precise the notion of Stein discrepancy alluded to in \cref{subsec:Stein Discrepancy}.

\begin{definition}[Stein Discrepancy \citep{gorham2015measuring}]
    Let $P$, $Q$ be probability distributions on a set $\cal{X}$.
    A Stein discrepancy is defined via a real-valued operator $\cal{T}_P$, which acts on a set $\cal{G}$ of $\R^d$-valued functions with domain $\cal{X}$, such that
we have a Stein identity:
\begin{equation}\label{eq:stein zero exp}
    \mathbb{E}_{X \sim P}[(\cal{T}_Pg)(X)] = 0 \quad \text{ for all }g \in \cal{G}.
\end{equation}
The Stein discrepancy is then defined as
\begin{equation}\label{eq:Stein Discrepancy}
    S(Q, P; \cal{G}; \cal{T}_P) := \sup_{g \in \cal{G}} \abs{\mathbb{E}_{X \sim Q}[(\cal{T}_Pg)(X)]}.
\end{equation}
\end{definition}
Stein discrepancies take the form as in \cref{eq:max expected deviation}, where $\cal{F}$ is designed to ``zero out'' any discrepancy under $P$, i.e., $\mathbb{E}_{X \sim P}[f(X)] = 0$ for all $f \in \cal{F}$. We see that by using \cref{eq:stein zero exp} and calculate
\begin{equation}
\begin{aligned}
    &\sup_{g \in \cal{G}} \abs{\mathbb{E}_{X \sim Q}[(\cal{T}_Pg)(X)]} \\
    &= \sup_{g \in \cal{G}} \abs{\mathbb{E}_{X \sim Q}[(\cal{T}_Pg)(X)] - \mathbb{E}_{X \sim P}[(\cal{T}_Pg)(X)]} \\
    &=  d_{\cal{T}_P\cal{G}}(Q,P),
\end{aligned}
\end{equation}
where $\cal{T}_P\cal{G} := \set{\cal{T}_Pg \in \Gamma(\cal{X}) \mid g \in \cal{G}}$. The operator $\cal{T}_P$ is called the \emph{Stein operator}, while $\cal{G}$ is referred to as the \emph{Stein set}. The Stein discrepancy is an integral probability metric (IPM) if and only if $\cal{T}_P\cal{G}$ forms a valid class of test functions for an IPM. Stein discrepancies are particularly attractive because they are often more computationally tractable than general integral probability metrics. One of their key advantages is flexibility: the choice of Stein operator $\cal{T}_P$ can be tailored to the specific problem at hand. For instance, in many applications, $\cal{T}_P$ can be defined in terms of the score function of $P$, which is especially useful when only an unnormalized density of $P$ is available. Additionally, unlike some statistical divergences that require explicit density estimates, the Stein discrepancy depends on $Q$ only through expectations, allowing it to be computed even when $Q$ is represented by an empirical sample. Stein’s method has led to significant advances in computational statistics in recent years. For a review of applications we refer to \citep{anastasiou2023stein}.

\subsection{Kernel Deviation from Stationarity}\label{app:KDS}

The general explicit form of the KDS is given as
\begin{equation}
    \begin{aligned}
        \mathcal{L}_1 \mathcal{L}_2 k(x,y)
        &=
        b(x)\cdot \nabla_x\nabla_{y} k(x,y) \cdot b(y) \\
        &\quad + \tfrac{1}{2}\, b(x)\cdot \nabla_x
        \operatorname{tr}\!\big(a(y) \nabla_{y}\nabla_{y} k(x,y)\big) \\
        &\quad + \tfrac{1}{2}\, b(y)\cdot \nabla_{y}
        \operatorname{tr}\!\big(a(x) \nabla_{x}\nabla_{x} k(x,y)\big) \\
        &\quad + \tfrac{1}{4}\,
        \operatorname{tr}\!\Big(a(x)
        \;\nabla_x\nabla_x^\top\;
        \operatorname{tr}\!\big(a(y) \nabla_{y}\nabla_{y} k(x,y)\big)\Big).
    \end{aligned}
\end{equation}
Note that fourth-order derivatives of the kernel are introduced. The general explicit form of the SKDS \cref{eq:simple closed form} only introduces derivatives of $k$ up to order two.

\section{PROOFS}\label{Supp:PROOFS}
\subsection{Proof of \texorpdfstring{\cref{lem:Representer function SKDS}}{Lemma~\ref{lem:Representer function SKDS}}}
\label{proof:representer function}

\begin{proof}
    We first show that if $\E_{X \sim \mu}[\mathcal{S} g(X)]$ is a continuous linear functional on $\cal{H}^d$, then, by the Riesz representation theorem, there exists a unique $g_{S, \mu} \in \cal{H}^d$ such that \cref{eq:SKDS representation function} holds for all $g \in \cal{H}^d$. An explicit form of $g_{\mathcal{S}, \mu}$ is obtained by using the reproducing property,
    \begin{equation}\label{eq:representer function SKDS ith component}
        \inner{g_{\mathcal{S}, \mu}(y)}{e_i} = \inner{g_{\mathcal{S}, \mu}}{K(\cdot, y) e_i}_{\cal{H}^d} = \E_{X \sim \mu}[\mathcal{S}_1 (K(X, y) e_i)],
    \end{equation}
    where $e_i$ is the $i$-th unit vector of $\R^d$. Hence, by the definition of the SKDS operator,
    \begin{equation}\label{eq:representer function SKDS}
        g_{\mathcal{S}, \mu}(y) = \E_{X \sim \mu}[\mathcal{S}_1 K(X, y)].
    \end{equation}
    
    We verify that $\E_{X \sim \mu}[\mathcal{S} g(X)]$ is a continuous linear functional. As it is the concatenation of an expectation and $\mathcal{S}$, which are both linear, it remains to show that $\E_{X \sim \mu}[\mathcal{S} g(X)]$ is continuous on $\cal{H}^d$. This is the case if $\E_{X \sim \mu}[\mathcal{S} g(X)]$ is bounded on the unit ball of $\cal{H}^d$. By Jensen's inequality and Cauchy-Schwarz,
    \begin{align*}
        \abs{\E_{X \sim \mu}[\mathcal{S} g(X)]} &= \abs{\E_{X \sim \mu}[2\inner{b(X)}{g(X)} + \operatorname{tr}\left(\sigma(X)\sigma(X)^T\nabla_x g(X)\right)]}\\
        &\leq \E_{X \sim \mu}[2\norm{b(X)}_2 \norm{g(X)}_2 + \norm{\sigma(X)\sigma(X)^T}_F \norm{\nabla_x g(X)}_F].
    \end{align*}
    We leverage the reproducing property \cref{eq:R^d valued kernel reproducing property} and bound
    \begin{equation}\label{eq:RKHS function bounded}
        \begin{aligned}
            \norm{g(x)}_2^2 &= \sum_{i \in [d]} \abs{\inner{g}{k(\cdot, x) e_i}_{\cal{H}^d}}^2 \\
            &\leq \norm{g}_{\cal{H}^d}^2 \sum_{i \in [d]} \norm{k(\cdot, x) e_i}_{\cal{H}^d}^2
        = d \norm{g}_{\cal{H}^d}^2 k(x, x).
        \end{aligned}
    \end{equation}
    Moreover, we use \citep[Cor. 4.36]{christmann2008support} to upper bound $\norm{\nabla_x g(X)}_F$, as
    \begin{equation}\label{eq:RKHS function grad bounded}
        \begin{aligned}
            \norm{\nabla_x g(x)}_F^2 &= \sum_{i,j \in [d]} \abs{\partial_j g_i(x)}^2\\
            &\leq \sum_{i,j \in [d]} \norm{g_i}_{\cal{H}}^2 \partial_{x_j} k(x,x)
            = d \norm{g}_{\cal{H}^d}^2 \sum_{j \in [d]} \partial_{x_j} k(x,x).
        \end{aligned}
    \end{equation}
    Then the boundedness of $\E_{X \sim \mu}[\mathcal{S} g(X)]$ on the unit ball of $\cal{H}^d$ follows from plugging in the bounds and using the assumptions on the integrands being square integrable with respect to $\mu$,
    \begin{align*}
        \abs{\E_{X \sim \mu}\left[\mathcal{S} g(X)\right]} &\leq \sqrt{d} \norm{g}_{\cal{H}^d} \E_{X \sim \mu}\left[2\norm{b(X)}_2 \sqrt{k(X,X)} + \norm{\sigma(X)\sigma(X)^T}_F \sqrt{\sum_{j \in [d]} \partial_{x_j} k(x,x)}\right].
    \end{align*}
    We verify that the expectation is indeed finite under the integrability assumptions. Using Cauchy-Schwarz inequality, we get
    \begin{align*}
        \E_{X \sim \mu}\left[2\norm{b(X)}_2 \sqrt{k(X,X)}\right] &\leq 2\E_{X \sim \mu}\left[\norm{b(X)}_2^2 \right] \E_{X \sim \mu}\left[k(X,X)\right],
    \end{align*}
    which is finite as $b \in L^2_\mu$ and $k \in L^2_\mu$. In the same way we calculate
    \begin{align*}
        \E_{X \sim \mu}\left[\norm{\sigma(X)\sigma(X)^T}_F \sqrt{\sum_{j \in [d]} \partial_{x_j} k(x,x)}\right] &\leq \E_{X \sim \mu}\left[\norm{\sigma(X)\sigma(X)^T}_F^2\right] \E_{X \sim \mu}\left[\sum_{j \in [d]} \partial_{x_j} k(x,x)\right].
    \end{align*}
    The term on the RHS is again finite, because $\sigma \in L_\mu^2$ and $\partial_{x_j} k(x,x)\in L_\mu^1$.
\end{proof}

\subsection{Proof of \texorpdfstring{\cref{lem:SKDS closed form}}{Lemma~\ref{lem:SKDS closed form}}}
\label{proof:SKDS closed form}
\begin{proof}
    The supremum of $\E_{X \sim \mu}\left[\mathcal{S} g(X)\right]$ over the unit ball $\cal{H}_{\leq 1}^d$ can be expressed in terms of the representer function $g_{\mathcal{S}, \mu}$ as
    \begin{align*}
        \sup_{g \in \cal{H}_{\leq 1}^d}\E_{X \sim \mu}\left[\mathcal{S} g(X)\right] = \sup_{g \in \cal{H}_{\leq 1}^d}\inner{g}{g_{\mathcal{S}, \mu}}_{\cal{H}^d} = \inner{\frac{g_{\mathcal{S}, \mu}}{\norm{g_{\mathcal{S}, \mu}}_{\cal{H}^d}}}{g_{\mathcal{S}, \mu}}_{\cal{H}^d} = \norm{g_{\mathcal{S}, \mu}}_{\cal{H}^d}.
    \end{align*}
    Using the representer property of $g_{\mathcal{S}, \mu}$ from \cref{eq:SKDS representation function} and plugging in the explicit form of $g_{\mathcal{S}, \mu}$ given in \cref{eq:representer function SKDS}, we obtain
    \begin{align*}
        \norm{g_{\mathcal{S}, \mu}}_{\cal{H}^d}^2 &= \inner{g_{\mathcal{S}, \mu}}{g_{\mathcal{S}, \mu}}_{\cal{H}^d} \\
        &= \E_{X \sim \mu}\left[\mathcal{S} g_{\mathcal{S}, \mu}(X)\right]\\
        &= \E_{X \sim \mu}\left[\mathcal{S}_1 \left[ \E_{Y \sim \mu}\left[ \mathcal{S}_2 K(X, Y)\right] \right]\right].
    \end{align*}
    Here, the subscript in $\mathcal{S}_1$ indicates that the operator is applied to the first and $\mathcal{S}_2$ to the second argument. We seek to interchange the expectation $\E_{Y \sim \mu}$ and application of $\mathcal{S}_1$. By \citep{wiki:BochnerIntegral} and the $\mu$-integrability of $\mathcal{S}_2 K(X, Y)$, the integration and the application of $\mathcal{S}_1$ may be interchanged if $\mathcal{S}_1$ is a continuous linear operator on the unit ball $B$ of $C_b^1(\R^d, \R^d)$ with respect to the norm $\norm{\cdot}_b = \sup_{x \in \R^d} \norm{\cdot}_2 + \sup_{x \in \R^d} \norm{\grad_x \cdot}_F$. Linearity is clearly satisfied, so we only need to show continuity.
    
    Let $f \in C_b^1(\R^d, \R^d)$. The boundedness of $\mathcal{S}$ on $B$ follows by the boundedness of $b$ and $\sigma$ as
    \begin{align*}
        \abs{\mathcal{S} f} \leq 2\norm{b(x)}_2\norm{f(x)}_2 + \norm{\sigma(x)\sigma(x)^T}_F \norm{\nabla_x f(x)}_F \leq \norm{f}_b\left(2\norm{b(x)}_2 + \norm{\sigma(x)\sigma(x)^T}_F\right).
    \end{align*}
    Hence, we can write
    \begin{align*}
        \norm{g_{\mathcal{S}, \mu}}_{\cal{H}^d}^2 &= \E_{X \sim \mu}\left[\E_{Y \sim \mu}\left[ \mathcal{S}_1 \mathcal{S}_2 K(X, Y)\right] \right].
    \end{align*}
\end{proof}

\subsection{Proof of \texorpdfstring{\cref{eq:simple closed form}}{Equation~\ref{eq:simple closed form}}}
\label{proof:simple closed form}
\begin{proof}
Let $K_y$ denote the kernel function $K(\cdot, y)$ for a fixed $y$. By the definition of $\mathcal{S}$, we can write
\begin{align*}
    (\mathcal{S} K_y(x))_i &= \inner{2b(x)}{K_y(x)_i} + \operatorname{tr}\left(\sigma(x)\sigma(x)^T\nabla K_y(x)_i\right)\\
    &= \inner{2b(x)}{e_i k(x,y)} + \operatorname{tr}\left(\sigma(x)\sigma(x)^T\nabla (e_i k(x,y)))\right)\\
    &= 2 b_i(x) k_y(x) + \operatorname{tr}\left(a(x)\left(e_i \nabla k_y(x)^T\right)\right)\\
    &= 2 b_i(x) k_y(x) + \operatorname{tr}\left(a_i(x)\nabla k_y(x)^T\right)\\
    &= 2 b_i(x) k_y(x) + \inner{a_i(x)}{\nabla k_y(x)},
\end{align*}
where $a_i$ is the $i$-th column of $a(x) =\sigma(x)\sigma(x)^T$. Thus, we can express the full operator as
\begin{equation}\label{eq:SKDS operator for scalar-valued kernel}
    \mathcal{S} K_y(x) = 2 b(x) k_y(x) + a(x)^T \nabla k_y(x).
\end{equation}
With this, $\mathcal{S}_1 \mathcal{S}_2 K(x, y)$ takes the closed form
\begin{equation}\label{eq:SKDS closed form scalar kernel}
\begin{aligned}
    \mathcal{S}_1 \mathcal{S}_2 K(x, y) &= \mathcal{S}_1 \left[ 2 b(y) k(x,y) + a(y)^T \nabla_y k(x, y) \right] \\
    &= 4 \inner{b(x)}{b(y) k(x,y)}\\
    &\quad + 2 \inner{b(x)}{a(y)^T \nabla_y k(x, y)} + 
 2 \operatorname{tr}\left(a(x) b(y) \nabla_x^T k(x,y)\right)\\
    &\quad + \operatorname{tr}\left(a(x) a(y)^T \nabla_x \nabla_y k(x, y)\right).
\end{aligned}
\end{equation}
By the trace rules we have
$$\operatorname{tr}\left(a(x) b(y) \nabla_x^T k(x,y)\right) = \operatorname{tr}\left(b(y) \nabla_x^T k(x,y) a(x)\right) = \inner{b(y)}{a(x)^T \nabla_x k(x,y)},$$
and the statement follows.
\end{proof}

\subsection{Proof of \texorpdfstring{\cref{prop:SKDS upper bound}}{Proposition~\ref{prop:SKDS upper bound}}}
\label{app: proof of SKDS upper bound}
\begin{proof}
    In the first part of the proof, we verify the Stein identity for the diffusion Stein operator $\mathcal{D}$ from \cref{eq:diffusion stein operator}, i.e.,
    \begin{equation}
        \E_{Z \sim P}[\mathcal{D}^{a+c}_p g(Z)] = 0,
    \end{equation}
    for all $g \in \cal{H}_{\leq 1}^d$ \citep[Prop. 3]{gorham2017measuring}. First, we establish that properties of the kernel $K$, or $k$, are inherited by functions in their respective RKH spaces. Specifically, by \citep[Cor. 4.23 and Cor. 4.36]{christmann2008support}, $h \in \cal{H}$ is bounded  and twice continuously differentiable with bounded derivatives, and hence, in specific, $h$ has Lipschitz continuous gradients. These bounds depend only on $k$ and $\norm{h}_{\cal{H}}$ as
    \begin{align}
        \abs{h(x)} &\leq \norm{h}_\cal{H} \norm{k}_\infty, \\
        \abs{\partial^i h(x)} &\leq \norm{h}_\cal{H} \sqrt{\partial^{i,i} k(x,x)}, \\
        \abs{\partial^j\partial^k h(x)} &\leq \norm{h}_\cal{H} \sqrt{\partial^{(j,k),(j,k)} k(x,x)}.
    \end{align}
    We define $\lambda_\cal{H} = \max_{i,j,k}\left(\norm{h}_\cal{H}, \sqrt{\partial^{i,i} k(x,x)}, \sqrt{\partial^{(j,k),(j,k)} k(x,x)}\right)$. As $g \in \cal{H}^d$ if and only if $g_i \in \cal{H}$ for all $i \in [d]$, the preceding arguments apply to any $g \in \cal{H}^d$. We conclude that $\cal{H}_{\leq 1}^d$, the unit ball of $\cal{H}^d$, subsets the (scaled) classical Stein-Set
    \begin{equation}
        \cal{G}_{\norm{\cdot}_2, \lambda_\cal{H}} = \set{ g:\R^d \ra \R^d \mid \sup_{x,y \in \R^d, x \neq y} \max \left( \norm{g(x)}_2, \norm{\grad_x g(x)}_2, \frac{\norm{\grad_x g(x) - \grad_x g(y)}_2}{\norm{x-y}_2}\right) \leq \lambda_\cal{H}}.
    \end{equation}
    Then, using \citep[Prop. 3]{gorham2017measuring}, it holds that $\E_{Z \sim P}[\mathcal{D}^{a+c}_p g(Z)] = 0$ for all $g \in \cal{H}_{\leq 1}^d$. The assumption that $p$ is supported on all of $\R^d$ implicitly enters the proof in \citep[Prop. 3]{gorham2017measuring} in an application of the divergence theorem when integrating over boundaries.
    \footnote{Alternatively to the full-support assumption, one could use the setting used in \citep[Proposition 1]{barp2019minimum}, where both $p$ and $\mu$ are continuously differentiable, and the kernel $K$ is assumed to be strictly integrally positive definite (IPD). The IPD condition ensures that $\E_{X \sim q, Z \sim q} [K(X,Z)] > 0$ for any non-zero measure $q$. In this setting one can weaken $\operatorname{supp}(p) = \R^d$ to $\mu > 0$ implies $p >0$.}

    Fix a joint distribution $P_{X,Z}$ on $(X,Z)$ such that the marginals are distributed as $X\sim \mu$ and $Z \sim p$. We then use the prior reasoning for a zero addition
    \begin{equation}
        E_{X \sim \mu}[\mathcal{S}_p g(X)] = \E_{(X,Z) \sim P_{X,Z}}[\mathcal{S}_p g(X) - \mathcal{D}^{a+c}_p g(Z)].
    \end{equation}
    By the definitions, and using the triangle inequality, Jensen's inequality and Cauchy-Schwarz, we get
    \begin{equation} \label{eq:expectation upper bound terms}
        \begin{aligned}
            \abs{E_{X \sim \mu}[\mathcal{S}_p g(X)]} & \leq \E\left[\abs{2 \inner{b(X)}{g(X)} + \operatorname{tr}(a(X)\grad_x g(X)) - 2 \inner{b(Z)}{g(Z)} - \operatorname{tr}(a(Z)\grad_x g(Z))}\right] \\
            &\quad+ \abs{\E[\operatorname{tr}(c(Z)\grad_x g(Z))]}\\
            &\leq \E[2 \abs{\inner{b(X)}{g(X)-g(Z)}} + 2 \abs{\inner{b(X) - b(Z)}{g(Z)}} \\
            &\quad+ \abs{\inner{\vect{a(X)^T}}{\vect{\grad_x g(X) - \grad_x g(Z)}}}
            + \abs{\inner{\vect{a(X)^T} - \vect{a(Z)^T}}{\vect{\grad_x g(Z)}}}]\\
            &\quad  + \abs{\E[\operatorname{tr}(c(Z)\grad_x g(Z))]}\\
            &\leq 2  \E[\norm{b(X)}_2 \norm{g(X)-g(Z)}_2] + 2  \E[\norm{b(X) - b(Z)}_2 \norm{g(Z)}_2] \\
            &\quad + \E[\norm{\vect{a(X)}}_2 \norm{\vect{\grad_x g(X) - \grad_x g(Z)}}_2] +  \E[\norm{\vect{a(X)} - \vect{a(Z)}}_2 \norm{\vect{\grad_x g(Z)}}_2] \\
            &\quad + \abs{\E[\operatorname{tr}(c(Z)\grad_x g(Z))]},
        \end{aligned}
    \end{equation}
    where we write the trace as a vector product, i.e., $\operatorname{tr}(A B) = \inner{\vect{A^T}}{\vect{B}}$.
    
    Building on the notation in \citep{gorham2017measuring}, we define for any function $g \in \Gamma(\R^d, \R^d)$ the constants $M_0(g) = \sup_{x \in \R^d} \norm{g(x)}_2$ and $M_1(g) = \sup_{x \neq y} \norm{g(x) - g(y)}_2 / \norm{x-y}_2$.
    
    The individual terms in \cref{eq:expectation upper bound terms} can be bounded as follows. By Cauchy-Schwarz and the definition of $M_0(g)$ and $M_1(g)$, we obtain
    \begin{equation*}
        \E[\norm{b(X)}_2 \norm{g(X)-g(Z)}_2] \leq \E[2 M_0(b) M_0(g)]
    \end{equation*}
    and
    \begin{equation*}
        \E[\norm{b(X)}_2 \norm{g(X)-g(Z)}_2] \leq \E[M_0(b) M_1(g) \norm{X-Z}_2].
    \end{equation*}
    We combine these two upper bounds and use the fact that $\min(a,b) \leq \sqrt{ab}$ for $a$,$b\geq 0$. Then,
    \begin{equation*}
        \begin{aligned}
        \E[\norm{b(X)}_2 \norm{g(X)-g(Z)}_2] &\leq  M_0(b) \E[\min(2M_0(g), M_1(g)\norm{X-Z}_2)] \\
        &\leq M_0(b) \E[\sqrt{ 2M_0(g)M_1(g)\norm{X-Z}_2}]\\
        &\leq M_0(b) \sqrt{ 2M_0(g)M_1(g)} \sqrt{\E[\norm{X-Z}_2]}.
    \end{aligned}
    \end{equation*}
    Similarly, we have
    \begin{equation*}
        \E[\norm{\vect{a(X)^T}}_2 \norm{\vect{\grad_x g(X) - \grad_x g(Z)}}_2] \leq M_0(\vect{a})\sqrt{ 2M_0(g)M_1(\vect{\grad_x g})} \sqrt{\E[\norm{X-Z}_2]},
    \end{equation*}
    as well as
    \begin{equation*}
            \E[\norm{b(X) - b(Z)}_2 \norm{g(Z)}_2] \leq M_1(b)M_0(g) \E[\norm{X-Z}_2],
    \end{equation*}
    and
    \begin{equation*}
        \E[\norm{\vect{a(X)} - \vect{a(Z)}}_2 \norm{\vect{\grad_x g(Z)}}_2] \leq M_1(\vect{a})M_0(\vect{\grad_x g}) \E[\norm{X-Z}_2].
    \end{equation*}
    We denote 
    \begin{align*}
    \lambda_k & = \sup_{g \in \cal{H}_{\leq 1}^d} \max (M_0(g), M_1(g), M_0(\vect{\grad_x g}, M_1(\vect{\grad_x g})), \text{ and } \\
    \lambda_{b, \sigma} & = 2^{\frac{3}{2}} \max(M_0(b), M_1(b), M_0(\vect{a}), M_1(\vect{a})).
    \end{align*}
    The finiteness of $\lambda_k$ follows from the arguments in the proof of \citep[Prop. 9]{gorham2017measuring}. The finiteness of $\lambda_{b, \sigma}$ follows from the boundedness assumptions on the respective functions and their derivatives.  As the inequalities hold for any $g \in \cal{H}_{\leq 1}^d$, we take the supremum over $\cal{H}_{\leq 1}^d$ on both sides to receive
    \begin{equation}\label{eq:raw upper wasser 2-bound}
        \begin{aligned}
            \sqrt{\operatorname{SKDS}(\mathcal{S}, \mu; \cal{H}_{\leq 1}^d)} \leq 2 \lambda_k \lambda_{b, \sigma} (\sqrt{\E[\norm{X-Z}_2]} + \E[\norm{X-Z}_2]) + \sup_{g \in \cal{H}_{\leq 1}^d}\E[\operatorname{tr}(c(Z)\grad_x g(Z))].
        \end{aligned}
    \end{equation}
    The stated inequality now follows by taking the infimum of these bounds over all joint distributions $(X, Z)$ with $X \sim \mu$ and $Z \sim p$, and by the definition of the Wasserstein-$2$ distance $d_{\mathcal{W}_2}(\mu, P) = \inf_{Z \sim P, X \sim \mu} \E[\norm{X - Z}_2]$.
    The supremum involving the stream-coefficient $c$ is upper bounded in \cref{lem:c-term vanishing}.
\end{proof}
    
The second term on the right-hand side of \cref{eq:raw upper wasser 2-bound} admits a closed form similarly to the SKDS due to the linearity in $g$. This can be seen by following the steps of \cref{lem:Representer function SKDS}, \cref{lem:SKDS closed form} and \cref{eq:simple closed form}. In specific, Riesz' representation theorem is applicable due to the boundedness of $k$ and $c$. However, as shown in the following lemma, under the assumption of a sufficiently rich RKHS, there is another, more interpretable, closed form of $\sup_{g \in \cal{H}_{\leq 1}^d}\E[\operatorname{tr}(c(Z)\grad_x g(Z))]$.

The following Lemma combines several properties of RKHS from \citep[Chapter 4.3]{christmann2008support}.
\begin{lemma}\label{lem:integral operator}
    Let $\cal{H}^d$ be the RKHS with kernel $K = k I_d$ for an universal kernel $k\in C_b(\R^d \times \R^d)$. Moreover, let $p$ be a probability density on $\R^d$ with respect to the Lebesgue measure. Then
    \begin{enumerate}
        \item $\cal{H}^d$ is dense in $L^2_p(\R^d, \R^d)$ with $\cal{H}^d \subseteq L^2_p(\R^d, \R^d)$.
        \item The inclusion operator $\operatorname{id}: \cal{H}^d \to L^2_p(\R^d, \R^d)$, $\operatorname{id}(g) = g$, is bounded, i.e., $\norm{g}_{L^2_p(\R^d, \R^d)} \leq C \norm{g}_{\cal{H}^d}$ for all $g \in \cal{H}^d$.

        The adjoint operator of the inclusion $\operatorname{id}: \cal{H}^d \to L^2_p(\R^d, \R^d)$ is defined by
        \begin{equation*}
            S_k^\star : L^2_p(\R^d, \R^d) \to \mathcal{H}; \quad (S_k^\star g) (x) := \int_X K(x, x')\, g(x') \, p(x') \, dx'.
        \end{equation*}
        From the boundedness and linearity of $\operatorname{id}$ it follows that $S_k^\star$ is linear and bounded as well.
    \item Fix $f \in L^2_p(\R^d, \R^d)$. Then
    \begin{equation}
        \sup_{g \in \cal{H}_{\leq 1}^d} \inner{g}{f}_{L^2_p(\R^d, \R^d)} = 0 \iff f \equiv 0 \quad p-a.e.
    \end{equation}
    In specific, $S_k^\star$ is injective.
    \end{enumerate}
\end{lemma}
\begin{proof}
    \begin{enumerate}
        \item By assumption, $k$ is universal, which implies that $K$ is universal (see \cref{def:universal kernel} and the comment below). Hence, $\cal{H}^d$ is dense in $C_c(\R^d)$. As $p$ is a finite Borel measure on $\mathbb{R}^d$, it holds that $C_c(\mathbb{R}^d,\mathbb{R}^d)$ is dense in $L_p^2(\mathbb{R}^d,\mathbb{R}^d)$ with respect to $\norm{\cdot}_{L^2_p(\R^d, \R^d)}$. This gives the result. By assumption $k$ is bounded, then \citep[Theorem 4.23]{christmann2008support} gives that any $g \in \cal{H}^d$ is bounded with $\abs{g(x)} \leq \norm{g}_{\cal{H}^d} \norm{k}_\infty$. Hence, the inclusion operator $\operatorname{id}: \cal{H}^d \to L^2_p(\R^d, \R^d)$, $\operatorname{id}(g) = g$ is well-defined.
        \item By the previous arguments, it follows $\norm{g}_{L^2_p(\R^d, \R^d)} \leq \norm{g}_{\cal{H}^d} \norm{k}_\infty$, and hence $\operatorname{id}$ is bounded.
        \item ``$\Ra$'' For a fixed $f \in L^2_p(\R^d, \R^d)$ and using the adjoint, we can write
        \begin{equation*}
            \sup_{g \in \cal{H}^d_{\leq 1}} \langle g, f \rangle_{L^2_p(\R^d, \R^d)} 
            = \sup_{g \in \cal{H}^d_{\leq 1}} \langle g, S_k^\star f \rangle_{\mathcal{H}^d}
            = \norm{S_k^\star f}_{\mathcal{H}},
        \end{equation*}
        where the last step follows by the Cauchy--Schwarz inequality in $\mathcal{H}^d$. So the supremum is exactly the RKHS norm of $S_k^\star f$.
        
        Now use the relationship between the kernel of the adjoint and the range of the forward operator: 
        for any bounded operator between Hilbert spaces we have $\ker(S_k^\star) = \bigl(\operatorname{range} \operatorname{id} \bigr)^\perp$. Because $\mathcal{H}^d$ is dense in $L^2_p(\R^d, \R^d)$, 
        $\operatorname{range} \operatorname{id}$ is dense in $L^2_p(\R^d, \R^d)$, hence $\bigl(\operatorname{range} \operatorname{id} \bigr)^\perp = \{0\}$.
        Therefore $\ker(S_k^\star) = \{0\}$, so $S_k^\star f = 0$ implies $f=0$ $p$-a.e. Moreover, $\norm{S_k^\star f}_{\mathcal{H}} = 0$ implies $S_k^\star f=0$ and the statement follows.

        \noindent
        ``$\La$''  If $f=0$ $p$-a.e. then trivially $\langle g,f \rangle_{L^2} = 0$ for all $g$, so the supremum is $0$.
        \end{enumerate}
\end{proof}

\begin{lemma}\label{lem:c-term vanishing}
    Let $P$ be a probability distribution with continuously differentiable density $p$. Let $Z \sim P$, and let $c \in C_b^1(\R^d, \R^d \times \R^d)$ be $P$-integrable. Let $K = k I_d$ with $k \in C_b^{(1,1)}(\R^d \times \R^d, \R)$. Then
    \begin{equation*}
        0 \leq \sup_{g \in \cal{H}_{\leq 1}^d}\E[\operatorname{tr}(c(Z)\grad_x g(Z))] \leq \norm{k}_\infty \norm{c \grad \log p + \inner{\grad}{c}}_{L^2_p(\R^d, \R^d)}.
    \end{equation*}
    If $k$ is universal, then
    \begin{equation}\label{eq:non-reversible component with RKHS norm}
        \sup_{g \in \cal{H}_{\leq 1}^d}\E[\operatorname{tr}(c(Z)\grad_x g(Z))] = \norm{S_k^\star (c \grad \log p + \inner{\grad}{c})}_{\mathcal{H}^d},
    \end{equation}    
    with
    \begin{equation}\label{eq:non-reversible component and injectivity}
        \sup_{g \in \cal{H}_{\leq 1}^d}\E[\operatorname{tr}(c(Z)\grad_x g(Z))] = 0 \iff c \grad \log p + \inner{\grad}{c} = p^{-1} \inner{\grad}{p c} = 0.
    \end{equation}    
\end{lemma}
\begin{proof}
    Let $Z \sim P$ with density $p$ and let $g \in \cal{H}^d_{\leq 1}$. We apply integration by parts (IbP) on the term $\E[\operatorname{tr}(c(Z)\grad_x g(Z))]$. To verify that IbP can be applied, it suffices to show that $g c\in L^1_p$ and $\inner{\grad}{g c}\in L^1_p$ \citep{pigola2014global}. This becomes clear by the following arguments. For $g \in \cal{H}^d_{\leq 1}$, it follows that $g \in C_b^1$ by the boundedness and differentiability of $k$ (as before, e.g., \cref{eq:RKHS function bounded} and \cref{eq:RKHS function grad bounded}). Moreover it holds $c \in C_b^1$ by assumption.

    We use the notation defined in \cref{app:Vector Calculus} when we denote the $j$-th row of $c$ by $c^j$ and the  $j$-th column of $c$ by $c_j$. Calculate
    \begin{equation}\label{eq:c-term to L_p inner product}
        \begin{aligned}
            \E_{Z \sim P}[\operatorname{tr}(c(Z)\grad_x g(Z))] &= \sum_j \int_{\R^d}  c^j(z)\grad g_j(z) p(z) dz &\\
            &= - \sum_j \int_{\R^d} g_j(z) \inner{\grad}{p(z)(c^j(z))^T} dz & (\text{IbP}) \\
            &= \sum_j \int_{\R^d} g_j(z) \inner{\grad}{p(z)c_j(z)} dz \\
            &= \sum_j \int_{\R^d} g_j \left(\inner{\grad p}{c_j} + p \inner{\grad}{c_j} \right) dz& (\text{Vector Calculus \cref{app:Vector Calculus}}) \\
            &= \sum_j \int_{\R^d} g_j \left(\inner{\grad \log p}{c_j} + \inner{\grad}{c_j} \right)p dz\\
            &= \inner{g}{c \grad \log p + \inner{\grad}{c}}_{L^2_p(\R^d, \R^d)}.
        \end{aligned}
    \end{equation}

    As $g \in \cal{H}^d_{\leq 1}$ is bounded, we have $g \in L^2_p(\R^d, \R^d)$. Hence,
    \begin{equation}
        \begin{aligned}
            \sup_{g \in \cal{H}_{\leq 1}^d}\E[\operatorname{tr}(c(Z)\grad_x g(Z))] &= \sup_{g \in \cal{H}_{\leq 1}^d} \inner{g}{c \grad \log p + \inner{\grad}{c}}_{L^2_p(\R^d, \R^d)}\\
            &\leq \sup_{g \in \cal{H}_{\leq 1}^d} \norm{g}_{L^2_p(\R^d, \R^d)} \norm{c \grad \log p + \inner{\grad}{c}}_{L^2_p(\R^d, \R^d)}\\
            &\leq \sup_{g \in \cal{H}_{\leq 1}^d} \norm{g}_{\cal{H}^d} \norm{k}_\infty \norm{c \grad \log p + \inner{\grad}{c}}_{L^2_p(\R^d, \R^d)}\\
            &= \norm{k}_\infty \norm{c \grad \log p + \inner{\grad}{c}}_{L^2_p(\R^d, \R^d)}.
        \end{aligned}
    \end{equation}
    If $k$ is universal, then, using the adjoint as in \cref{lem:integral operator}, we calculate
    \begin{equation}
        \begin{aligned}
            \sup_{g \in \cal{H}_{\leq 1}^d}\E[\operatorname{tr}(c(Z)\grad_x g(Z))] &= \sup_{g \in \cal{H}_{\leq 1}^d}\inner{\operatorname{id}(g)}{c \grad \log p + \inner{\grad}{c}}_{L^2_p(\R^d, \R^d)}\\
            &= \sup_{g \in \cal{H}_{\leq 1}^d} \inner{\operatorname{id}(g)}{c \grad \log p + \inner{\grad}{c}}_{L^2_p(\R^d, \R^d)}\\
            &= \sup_{g \in \cal{H}_{\leq 1}^d} \inner{g}{S_k^\star (c \grad \log p + \inner{\grad}{c})}_{\mathcal{H}^d}\\
            &= \norm{S_k^\star (c \grad \log p + \inner{\grad}{c})}_{\mathcal{H}^d}.
        \end{aligned}
    \end{equation}
    An application of the previously obtained result \cref{lem:integral operator} concludes the proof.
\end{proof}
    
Interestingly, the term $c \grad \log p + \inner{\grad}{c}$ is (up to a constant factor) the non-reversible component mentioned in \cref{app:SDE-Generator-Martingale Problem}, and, when multiplied with $p$, resolves to the stationary probability flux \citep{pavliotis2014stochastic}.

\subsection{Proof of \texorpdfstring{\cref{th:SKDS consistency}}{Theorem~\ref{th:SKDS consistency}}}
\begin{proof}
Before we start with the proof, note that assumption (v), $C_c^\infty \subseteq \cal{H}$ implies the universality of $k$, because $C_c^\infty$ is dense in $C_c$.

\noindent
``$\Ra$'': We assume that $\mu$ is the stationary solution to the SDE. As the diffusion process is reversible, the detailed balance condition holds (see \cref{eq:detailed balance condition}). In specific, it holds
\begin{equation}
    \begin{aligned}
        0 &= \frac{1}{\mu} \inner{\grad}{\mu c}\\
        &= c \grad \log \mu + \inner{\grad}{c}.& (\text{Vector Calculus \cref{app:Vector Calculus}})
    \end{aligned}
\end{equation}
An application of \cref{prop:SKDS upper bound} with $d_{\cal{W}_2}(\mu, \mu) = 0$ and $\norm{c \grad \log p + \inner{\grad}{c}}_{L_\mu^2(\R^d, \R^d)} = 0$, concludes the proof.

\noindent
``$\La$'':
We first show that a vanishing SKDS implies that $\mu$ is the stationary density of the diffusion process $(X_t)_{t \geq 0}$. We therefore follow the arguments showing that the KDS characterizes stationary diffusions with their stationary density matching the target density \citep[Theorem 3]{Lorch:2024}. Note that (ii) and (iv) implies that $C_c^\infty(\R^d)$ is a core of the generator $\cal{A}$ of the stochastic process $(X_t)_{t \geq 0}$ \citep[Theorem 1.6, p.370]{ethier2009markov}. The generator takes the form $\cal{L} \cdot = \inner{b}{\grad \cdot} + \frac{1}{2} \operatorname{tr}\left({\sigma \sigma^T \grad\grad \cdot}\right)$ on $C^2(\R^d)$, and in specific on $C_c^\infty(\R^d)$. Hence a stationary solution $\mu$ is characterized by
\begin{equation}\label{eq:stationary solution characterization - SKDS proof}
    \sup_{u \in C_c^\infty(\R^d)} \E_{X \sim \mu}[\cal{L} u] = 0.
\end{equation}
Now, we have that $\cal{L} u = \frac{1}{2} \mathcal{S}[\grad u]$. From $\operatorname{SKDS}(\mathcal{S}, \mu; \cal{H}_{\leq 1}^d) = 0$ it follows that $\E_{X \sim \mu}[\mathcal{S} g] = 0$ for all $g \in \cal{H}^d_{\leq 1}$ since the supremum is non-negative. This expands to any $g\in\cal{H}$, as 
$$\E_{X \sim \mu}[\mathcal{S} g] = \norm{g}_{\cal{H}^d} \inner{g / \norm{g}_{\cal{H}^d}}{g_{\mathcal{S}, \mu}} = \norm{g}_{\cal{H}^d} \cdot 0 = 0,$$
where we used that $g / \norm{g}_{\cal{H}^d} \in  \cal{H}^d_{\leq 1}$.
By the assumptions on the kernel, (iii) and (v), we have that $C_c^\infty(\R^d, \R^d) \subseteq \cal{H}^d$. Equation \cref{eq:stationary solution characterization - SKDS proof} then follows from the observation that $\set{\grad_x u \mid u \in C_c^\infty(\R^d)} \subseteq C_c^\infty(\R^d, \R^d)$.

It remains to show that $(X_t)_{t \geq 0}$ is a \emph{reversible} stationary diffusion. We use the Stein identity of the diffusion Stein operator $\cal{D}_\mu^{a+c}$ in the same way as before in \cref{app: proof of SKDS upper bound}, i.e., $\E_{Z \sim \mu}[\mathcal{D}^{a+c}_\mu g(Z)] = 0$ for all $g \in \cal{H}_{\leq 1}^d$. Then, we calculate
\begin{equation}\label{eq:SKDS zero-addition under stationarity}
    \begin{aligned}
        E_{X \sim \mu}[\mathcal{S}_\mu g(X)] &= \E_{X \sim \mu}[\mathcal{S}_\mu g(X)] - \E_{Z \sim \mu}[\mathcal{D}^{a+c}_\mu g(Z)]\\
        &= \E_X\left[2 \inner{b(X)}{g(X)} + \operatorname{tr}(a(X)\grad_x g(X))\right]\\
        &\quad - \E_Z\left[2 \inner{b(Z)}{g(Z)} + \operatorname{tr}(a(Z)\grad_x g(Z)) + \operatorname{tr}(c(Z)\grad_x g(Z))\right]\\
        &= - \E_Z\left[\operatorname{tr}(c(Z)\grad_x g(Z))\right].
    \end{aligned}
\end{equation}
This sets us up for
\begin{equation}
    \begin{aligned}
        0 &= \sqrt{\operatorname{SKDS}(\mathcal{S}, \mu; \cal{H}_{\leq 1}^d)} \\
        &= \sup_{g \in \cal{H}_{\leq 1}^d} E_{X \sim \mu}[\mathcal{S}_\mu g(X)] \\
        &= \sup_{g \in \cal{H}_{\leq 1}^d} \E_Z\left[\operatorname{tr}(c(Z)\grad_x g(Z))\right] & \cref{eq:SKDS zero-addition under stationarity}\\
        &= \norm{S_k^\star (c \grad \log p + \inner{\grad}{c})}_{\mathcal{H}^d}. & \cref{eq:non-reversible component with RKHS norm}
    \end{aligned}
\end{equation}
We conclude $\mu^{-1} \inner{\grad}{\mu c} = 0$ using \cref{lem:c-term vanishing} (specifically \cref{eq:non-reversible component and injectivity}), and hence $(X_t)_{t \geq 0}$
is a reversible diffusion process with stationary density $\mu$.
\end{proof}

\subsection{Proof of \texorpdfstring{\cref{cor:SKDS convex}}{Corollary~\ref{cor:SKDS convex}}}
    \begin{proof}
    Assume that we have a linear parametrization $\mathcal{S}^{\theta}\cdot = \inner{\theta}{\cal{T}\cdot}$. We aim to show that $$\operatorname{SKDS}(\mathcal{S}^\theta, \mu; \cal{H}_{\leq 1}^d) = (\sup_{g \in \cal{H}_{\leq 1}^d}\E_{X \sim \mu}\left[\mathcal{S}^\theta g(X)\right])^2$$ is convex in the parameters $\theta$. This follows from standard results on convexity preserving functions as follows. 
    Firstly, due to the linearity of the expectation, $$\E_{X \sim \mu}\left[\mathcal{S}^\theta g(X)\right] = \inner{\theta}{\E_{X \sim \mu}\left[\cal{T} g(X) \right]}$$ is linear in $\theta$. Then it holds generally that a point wise supremum of linear (convex) functions is convex. In other words, a function $f(x) = \sup_{y \in \mathcal{A}} h(x, y)$, which is convex in $x$ for each $y$, is convex in $x$. Lastly, squaring a non-negative convex function preserves convexity.
    Non-negativity follows by the linearity of the functional $\cal{T}$ (and the linearity of $\cal{S}^\theta$ respectively), as $$\E_{X \sim \mu}\left[\mathcal{S}^\theta (-g)(X)\right] = -\E_{X \sim \mu}\left[\mathcal{S}^\theta g(X)\right].$$ Since $g \in \cal{H}_{\leq 1}^d$ if and only if $-g \in \cal{H}_{\leq 1}^d$, the supremum over $\cal{H}_{\leq 1}^d$ is non-negative.
\end{proof}

\subsection{Calculations in \texorpdfstring{\cref{ex:linear SDE parametrization}}{Example~\ref{ex:linear SDE parametrization}}}

For $b(x) = B^\theta j(x)$, we calculate
\begin{align*}
    \inner{b(x)}{g(x)} &= g(x)^T B^\theta j(x)\\
    &= \operatorname{tr}\left(g(x)^T B^\theta j(x)\right)\\
    &= \operatorname{tr}\left(B^\theta j(x) g(x)^T \right)\\
    &= \mathrm{vec}\left((B^\theta)^T \right)^T \mathrm{vec}\left(j(x) g(x)^T\right) \\
    &= \inner{\mathrm{vec}\left((B^\theta)^T \right)}{\mathrm{vec}(j(x) \otimes g(x))},
\end{align*}
where $j(x) \otimes g(x) = j(x) g(x)^T$ denotes the outer product. We defined $V$ to stack the vectorized basis diffusion matrices, i.e., $V(x) = (\vect{v_1(x)v_1(x)^T}\mid \ldots \mid  \vect{v_m(x)v_m(x)^T})^T$. Then
\begin{align*}
    \operatorname{tr}(a(x)\grad_x g(x)) &= \operatorname{tr}\left(\sum_{i=1}^m A_i^\theta \, v_i(x)v_i(x)^T\grad_x g(x)\right)\\
    &= \sum_{i=1}^m A_i^\theta \operatorname{tr}\left((v_i(x)v_i(x)^T\grad_x g(x)\right)\\
    &= \sum_{i=1}^m A_i^\theta \operatorname{vec}\left(v_i(x)v_i(x)^T\right)^T\operatorname{vec}\left(\grad_x g(x)\right)\\
    &= \inner{\left( A_i^\theta \right)_{i \in [m]}}{V(x) \operatorname{vec}\left(\grad_x g(x)\right)}.
\end{align*}

\subsection{Proof of \texorpdfstring{\cref{prop:SKDS emp convex}}{Proposition~\ref{prop:SKDS emp convex}}}\label{app:Calculations for Example 10}
From \cref{ex:linear SDE parametrization} we have
\begin{equation}\label{eq:lin operator T}
    \mathcal{T}g(x) = 
    \begin{pmatrix}
    \operatorname{vec}\big(j(x) \otimes g(x)\big) \\
    V(x) \operatorname{vec}\big(\nabla g(x)\big)
    \end{pmatrix} \in \Gamma(\R^d, \R^{ld + m}),
\end{equation}
where $V$ stacks the vectorized basis diffusion matrices, i.e., $V(x) = (\vect{v_1(x)v_1(x)^T}\mid \ldots \mid  \vect{v_m(x)v_m(x)^T})^T$. We let $\cal{T}^{(i)}$ denote the projection onto the $i$-th entry and $\cal{T}_1$ indicate the application to the first argument. The operator $\cal{T}$, in an abuse of notation, is expanded to act on $C^1(\R^d, \R^{d \times d})$ as in \cref{def:SKDS Operator} defined for the SKDS operator $\mathcal{S}$.
 We define the following square matrices for the scope of this section;
\begin{equation}
    R := \left(\inner{\E_{X \sim \mu}\left[(\cal{T}_1^{(i)} K(X,\cdot))^T\right]}{\E_{X^\prime \sim \mu}\left[(\cal{T}_1^{(j)} K(X^\prime,\cdot))^T\right]}_{\mathcal{H}^d}\right)_{i,j \in [dl + m]}
\end{equation}
and define
\begin{equation}
    M(x,y) := \cal{T}_2 \left(\left(\cal{T}_1 K(x, y)\right)^T\right)
\end{equation}
as the consecutive application of $\mathcal{T}$ to both arguments of a matrix-valued kernel $K$. For notational convenience, we sometimes omit the arguments of $M(x,y)$ and only write $M$.
The proof of \cref{prop:SKDS emp convex} further requires the following two lemmas.
The next result shows that in $R$ the expectation and the inner product of $\cal{H}^d$ commute.
\begin{lemma}\label{lem:helper T operator}
    Let $\cal{T}$ be defined as in \cref{ex:linear SDE parametrization}, and let the basis functions $j$, $v_i v_i^T$ be bounded and continuous. In specific, $\norm{j}_\infty$, $\norm{v_i v_i^T}_\infty \leq C$ for all $i \in [m]$, and a $C \in [1, \infty)$. Moreover, let $K = k I_d$ with $k \in C_b^{(2,2)}(\R^d \times \R^d, \R)$. Then $\E_{X \sim \mu}{[\mathcal{T}g]}$ is a continuous linear functional on $\cal{H}^d$, and it holds that
    \begin{equation}\label{eq:matrix equality T operator}
            R = \E_{X \sim \mu}\left[\E_{Y \sim \mu}\left[ \cal{T}_2 \left(\left(\cal{T}_1 K(X,Y)\right)^T\right) \right] \right] = \E_{X \sim \mu}\left[\E_{Y \sim \mu}\left[ M(X,Y) \right] \right].
    \end{equation}
\end{lemma}
\begin{proof}
    The linearity follows immediately from the linearity of both the expectation and $\cal{T}$.
    We further show that the operator $\cal{T}$ is bounded on $\cal{H}_{\leq 1}^d$ as follows.
    The first component $\operatorname{vec}\big(j(x) \otimes g(x)\big)$ is bounded due to the boundedness assumptions on $j$ and $k(x,x)$ using \cref{eq:RKHS function bounded} since $\norm{j}_\infty \norm{g}_\infty \leq C \sup_{x \in \R^d} k(x,x)$.
    Boundedness of the second component $V(x) \operatorname{vec}\big(\nabla g(x)\big)$ follows similarly from the boundedness of $V$ and $\partial_{x_j} k(x,x)$ using \cref{eq:RKHS function grad bounded} as $\norm{v}^2_\infty \norm{\grad g}_\infty \leq C^2 \sup_{x \in \R^d} \sqrt{\partial_{x_j} k(x,x)}$.
    As $\norm{\E_{X \sim \mu}{[\mathcal{T}g]}}_2 \leq (dl + m)\norm{\mathcal{T}g}_\infty$, we obtain that $\E_{X \sim \mu}{[\mathcal{T}g]}$ is bounded on $\cal{H}_{\leq 1}^d$, and hence continuous. 
    
    By Riesz representation theorem, there exists $(g_{\mu, \cal{T}^{(i)}})_{i \in [dl + m]} \subset \cal{H}^d$ such that $\E_{X \sim \mu}{[\cal{T}^{(i)}g]} = \inner{g_{\mu, \cal{T}^{(i)}}}{g}_{\cal{H}^d}$. Following the same steps as in \cref{proof:representer function} we conclude
    \begin{equation}
        g_{\mu, \cal{T}^{(i)}}(\cdot) = \E_{X \sim \mu}{[\cal{T}_1^{(i)} K(X, \cdot)^T]}.
    \end{equation}
    Hence, $R_{ij} = \E_{X \sim \mu}{\left[\cal{T}^{(i)}\left(\E_{X^\prime \sim \mu}\left[(\cal{T}_1^{(j)} K(X^\prime,\cdot))^T\right]\right)\right]}$. Since $\cal{T}$ is a continuous linear operator, $\E$ and $\cal{T}$ commute, and the statement follows.
\end{proof}

The following lemma bounds the entries of $M$.
\begin{lemma}\label{lem: helper for Error Estimate for Matrix Sampling Estimators}
Under the assumptions of \cref{lem:helper T operator}, for all $x, y \in \R^d$ it holds that
\begin{align}
    \norm{M(x,y)}_F &\le L_1 := (dl + m)^2 d^2 C^2 C_k,\\
    m_2(M(x,y)) &:= \max\big(\norm{\E[M M^T]}_F, \norm{\E[M^T M]}_F\big) \le L_2 := (dl + m)^3 (d^2 C^2 C_k)^2,
\end{align}
where $m_2(M(x,y))$ denotes the per-sample second moment and $C_k$ is a constant dependent on $k$.
\end{lemma}

\begin{proof}
    
    To bound $\norm{M(x,y)}_F$,  we explicitly write down $M(x,y) = \cal{T}_2 \left(\left(\cal{T}_1 K(x, y)\right)^T\right)$ as
    \begin{equation*}
        \begin{aligned}
            \cal{T}_1 K(x, y) &= \left( \cal{T}_1 k(x, y)e_1 \mid \ldots \mid \cal{T}_1 k(x, y)e_d \right) \quad \text{with}\quad \cal{T}_1 k(x, y)e_i = \begin{pmatrix}
        \operatorname{vec}\big(j(x) \otimes k(x, y)e_i\big) \\
        V(x) \operatorname{vec}\big(e_i \otimes \nabla_x k(x, y)\big)
        \end{pmatrix}.
        \end{aligned}
    \end{equation*}
    The transpose of the full matrix is then
    \begin{equation*}
        \begin{aligned}
            \left(
        \mathcal{T}_1 K(x, y)
        \right)^T &=
            \left[
            \begin{array}{ccccccc}
            k(x, y) j(x)^T & 0 & \cdots & 0 & \multicolumn{1}{|c}{(v_1 v_1^T)_{1 \cdot} \nabla_x k(x,y)} & \cdots & (v_m v_m^T)_{1 \cdot} \nabla_x k(x,y) \\
            0 & k(x, y) j(x)^T & \cdots & 0 & \multicolumn{1}{|c}{(v_1 v_1^T)_{2 \cdot} \nabla_x k(x,y)} & \cdots & (v_m v_m^T)_{2 \cdot} \nabla_x k(x,y) \\
            \vdots & \vdots & \ddots & \vdots & \multicolumn{1}{|c}{\vdots} & \ddots & \vdots \\
            0 & 0 & \cdots & k(x, y) j(x)^T & \multicolumn{1}{|c}{(v_1 v_1^T)_{d \cdot} \nabla_x k(x,y)} & \cdots & (v_m v_m^T)_{d \cdot} \nabla_x k(x,y)
            \end{array}
            \right]\\
        &=  \left[\begin{array}{cccc|ccc}
            k(x, y) j(x)^T & 0 & \cdots & 0 & &\\
            0 & k(x, y) j(x)^T & \cdots & 0 & &\\
            \vdots & \vdots & \ddots & \vdots & &\\
            0 & 0 & \cdots & k(x, y) j(x)^T & &\\
            \end{array}
            \begin{array}{c|c|c}
            v_1 v_1^T \grad_x k(x,y) & \ldots & v_m v_m^T \grad_x k(x,y)
            \end{array}
            \right].
        \end{aligned}
    \end{equation*}
    When applying $\cal{T}_2$, there are two cases to distinguish. First, calculate
    \begin{equation*}
        \begin{aligned}
            \cal{T}_2 (k(x,y) j(x)_i e_j ) &= \begin{pmatrix}
                \operatorname{vec}\big(j(y) \otimes (k(x,y) j(x)_i e_j )\big) \\
                V(y) \cdot \operatorname{vec}\big(\nabla_y (k(x,y) j(x)_i e_j )\big)
            \end{pmatrix}\\
                &= j(x)_i \begin{pmatrix}
                k(x,y)\operatorname{vec}\big(  j(y) \otimes e_j \big) \\
                V(y) \cdot \operatorname{vec}\big(e_j  \otimes \nabla_y k(x,y) \big)
            \end{pmatrix},
        \end{aligned}
    \end{equation*}
    and for the second case
    \begin{equation*}
        \begin{aligned}
            \cal{T}_2 (v_i(x) v_i(x)^T \grad_x k(x,y)) &= \begin{pmatrix}
                \operatorname{vec}\big(j(y) \otimes (v_i(x) v_i(x)^T \grad_x k(x,y))\big) \\
                V(y) \cdot \operatorname{vec}\big(\nabla_y (v_i(x) v_i(x)^T \grad_x k(x,y))\big)
            \end{pmatrix}\\
            &= \begin{pmatrix}
                \operatorname{vec}\big(j(y) \otimes (v_i(x) v_i(x)^T \grad_x k(x,y))\big) \\
                V(y) \cdot \operatorname{vec}\big(v_i(x) v_i(x)^T \nabla_y \grad_x k(x,y)\big)
            \end{pmatrix}.
        \end{aligned}
    \end{equation*}
    We bound every entry to obtain
    \begin{align}
        \norm{j(x)_i k(x,y)\operatorname{vec}\big(  j(y) \otimes e_j \big)}_\infty &\leq C^2 \sup_{x,y \in \R^d}k(x,y),\\
        \norm{j(x)_i V(y) \cdot \operatorname{vec}\big(e_j  \otimes \nabla_y k(x,y) }_\infty  &\leq d C^2 \sup_{x,y \in \R^d}\norm{\nabla_y k(x,y)}_\infty,\\
        \norm{\operatorname{vec}\big(j(y) \otimes (v_i(x) v_i(x)^T \grad_x  k(x,y))\big)}_\infty
        &\leq d C^2 \sup_{x,y \in \R^d}\norm{\nabla_x k(x,y)}_\infty,\\
        \norm{V(y) \cdot \operatorname{vec}\big(v_i(x) v_i(x)^T \nabla_y \grad_x k(x,y)\big)}_\infty
        &\leq d^2 C^2 \sup_{x,y \in \R^d}\norm{\nabla_y \nabla_x k(x,y)}_\infty.
    \end{align}
    We define $C_k := \sup_{x,y \in \R^d} \max(k(x,y), \norm{\nabla_x k(x,y)}_\infty, \norm{\nabla_y k(x,y)}_\infty, \norm{\nabla_y \nabla_x k(x,y)}_\infty)$ and bound the Frobenius norm of $M(x,y)$ with the entrywise bound for each of the $(dl + m)^2$ entries. Then, it holds that
    \begin{equation}\label{eq:L_1}
        \norm{M(x,y)}_F \leq (dl + m)^2 d^2 C^2 C_k = L_1
    \end{equation}
    for all $x$, $y \in \R^d$.
    We further bound the per-sample second moment by
    \begin{equation}\label{eq:L_2}
        m_2(M(x,y)) \leq (dl + m)^3 (d^2 C^2 C_k)^2 = L_2.
    \end{equation}
\end{proof}

\begin{proof}[Proof of \cref{prop:SKDS emp convex}]
    By \cref{lem:SKDS closed form} and substituting the linear parametrization, we have
    \begin{equation}
        \begin{aligned}
            \operatorname{SKDS}(\mathcal{S}^\theta, \mu; \cal{H}_{\leq 1}^d) &= \E_{X \sim \mu}\left[\E_{Y \sim \mu}\left[ \mathcal{S}^\theta_2 \mathcal{S}^\theta_1 K(X, Y)\right] \right]\\
            &= \E_{X \sim \mu}\left[\E_{Y \sim \mu}\left[ \inner{\theta}{\cal{T}_2\left(\inner{\theta}{\cal{T}_1(K(X, Y))}^T\right)}\right] \right]\\
            &= \inner{\theta}{\E_{X \sim \mu}\left[\E_{Y \sim \mu}\left[ \cal{T}_2 \left(\left(\cal{T}_1 K(x,y)\right)^T\right) \right] \right]\theta}\\
            &= \inner{\theta}{R \theta}.
        \end{aligned}
    \end{equation}
    Note that $R$ is a Gram-matrix, hence PSD. This is an alternative proof for the convexity of the SKDS, although with stronger regularity assumptions.

    Similarly, the empirical estimator in \cref{eq:empirical estimate SKDS} can be written using the identity $\theta^T M(x,y) \theta = \frac12 \theta^T (M(x,y) + M(x,y)^T) \theta$, which holds for any real vector $\theta$ and square matrix, yielding
    \begin{equation}
        \begin{aligned}
            \hat{\operatorname{SKDS}}(\mathcal{S}^{\theta}, D; \cal{H}^d_{\leq 1}) &= \frac{1}{\floor{N/2}} \sum_{n = 1}^{\floor{N/2}} \mathcal{S}_1^{\theta} \mathcal{S}_2^{\theta} K(x_{2n - 1}, x_{2n})\\
            &= \inner{\theta}{ \frac{1}{\floor{N/2}} \sum_{n = 1}^{\floor{N/2}} \cal{T}_1 \left(\left(\cal{T}_2 K(x_{2n - 1}, x_{2n})\right)^T\right) \theta}\\
            &= \inner{\theta}{ \frac{1}{\floor{N/2}} \sum_{n = 1}^{\floor{N/2}} \frac{1}{2}\left(\cal{T}_1 \left(\left(\cal{T}_2 K(x_{2n - 1}, x_{2n})\right)^T\right) + \cal{T}_1 \left(\left(\cal{T}_2 K(x_{2n - 1}, x_{2n})\right)^T\right)^T\right)\theta}\\
            &=: \inner{\theta}{\frac{1}{\floor{N/2}} \sum_{n = 1}^{\floor{N/2}} R_n \theta}.
        \end{aligned}
    \end{equation}

    Note that $R_n$ are independent copies of $\frac{1}{2}(M + M^T)$. From \cref{lem:helper T operator}, we know $\E[M] = R$, and as $R = R^T$, we have $\E[M^T] = \E[M]^T = R$ as well. We conclude
    \begin{equation*}
        \E[R_n] = \E[\frac{1}{2}(M(x_{2n - 1}, x_{2n}) + M(x_{2n - 1}, x_{2n})^T)] = R,
    \end{equation*}
    with $(R_n)_{n \in \floor{N/2}}$ independent.
    Using \cref{lem: helper for Error Estimate for Matrix Sampling Estimators}, we apply the error estimation for matrix samples from \citep[Cor 6.2.1]{tropp2015introduction} and get for any $\epsilon > 0$ that
    \begin{equation}
        \begin{aligned}
            \P\left[\norm{\frac{1}{\floor{N/2}} \sum_{n = 1}^{\floor{N/2}} R_n - R}_F \geq \epsilon\right] \leq 2 (dl + m) \exp{\left( - \frac{\floor{N/2} \epsilon^2 / 2}{L_2 + 2 L_1 \epsilon/3}\right)},
        \end{aligned}
    \end{equation}
    for constants $L_1$ and $L_2$ as specified in \cref{lem: helper for Error Estimate for Matrix Sampling Estimators}.
    The Frobenius norm can be lower bounded with basic linear algebra and Weyl's inequality for spectral stability, yielding
    \begin{equation}\label{eq:spectral stability}
        \begin{aligned}
            \norm{\frac{1}{\floor{N/2}} \sum_{n = 1}^{\floor{N/2}} R_n - R}_F &\geq \norm{\frac{1}{\floor{N/2}} \sum_{n = 1}^{\floor{N/2}} R_n - R}_{\text{op}, 2}\\
            &\geq \max_{k \in [dl + m]} \abs{\lambda_{(k)}\left(\frac{1}{\floor{N/2}} \sum_{n = 1}^{\floor{N/2}} R_n\right) - \lambda_{(k)}(R)},
        \end{aligned}
    \end{equation}
    where $\lambda_{(k)}(\cdot)$ denotes the $k$-th eigenvalue. Note that Weyl's inequality is only applicable to symmetric (or hermitian) matrices. We argued that $R$ is PSD and, in specific, symmetric. The estimation matrix $R_n$ is constructed to be symmetric as well. We conclude from \cref{eq:spectral stability} and the positive semi-definiteness of $R$ that the empirical estimator $\hat{\operatorname{SKDS}}(\mathcal{S}^{\theta}, D; \cal{H}^d_{\leq 1})$ is $\epsilon$-strongly quasiconvex with high probability, In other words, it holds that
    \begin{equation}
        \hat{\operatorname{SKDS}}(\mathcal{S}^{\theta}, D; \cal{H}^d_{\leq 1}) + \epsilon \norm{\theta}_2
    \end{equation}
    is convex in $\theta$ with probability greater than $1 - 2 (dl + m) \exp{\left( - \frac{\floor{N/2} \epsilon^2 / 2}{L_2 + 2 L_1 \epsilon/3}\right)}$. The result follows from rearranging for $N$.
\end{proof}

\section{EXPERIMENTS}
\subsection{Kernel Choices}\label{app:Kernel Choices}
In this work, we use the following kernels for model construction and benchmarking. Each kernel emphasizes different properties of the input space or adjusts standard kernels to better suit our setting.

\begin{enumerate}
    \item \textbf{RBF Kernel:} 
    \[
        k_{\text{RBF}}(x,y) = \exp{\left(- \frac{\norm{x-y}_2^2}{2 \sigma^2} \right)}.
    \]
    This standard Gaussian (Radial Basis Function) kernel promotes smoothness by measuring similarity based on the Euclidean distance between inputs. It is widely used for its universal approximation properties and strong empirical performance.

    \item \textbf{Tilted RBF Kernel:}
    \[
        k_{\text{RBF, tilted}}(x,y) = \frac{1}{w_1(x) w_1(y)}\exp{\left(- \frac{\norm{x-y}_2^2}{2 \sigma^2} \right)}.
    \]
    This variant of the RBF kernel downweights regions far from the origin using the weight function \( w_1(x) = (1 + \|x\|_2^2)^{1/2} \). It effectively biases the kernel toward the origin and suppresses the influence of distant samples, which can be beneficial in high-dimensional or heavy-tailed settings.

    \item \textbf{IMQ+ Kernel:}
    \[
        k_{\text{IMQ}^+}(x,y) = \frac{1}{w_1(x) w_1(y)}\left(\frac{1}{w_1(x-y)} +  \left(1 +  \inner{x}{y}\right)\right),
    \]
    where \( w_1(x) = (1 + \|x\|_2^2)^{1/2} \). This kernel combines a repulsive inverse multiquadric term and an attractive linear term, both scaled by polynomially decaying weights. It was proposed in \citep[Corollary 3.4]{kanagawa2022controlling} and is designed to control the spectral decay of the associated integral operator, improving generalization in certain distributional regimes.
\end{enumerate}

\subsection{Compute Infrastructure}\label{subsec:Compute Infrastructure}
All experiments were conducted on a high-performance computing (HPC) cluster equipped with \textit{Intel(R) Xeon(R) Platinum 8480+ (Sapphire Rapids)} processors.  

Jobs were scheduled using \texttt{SLURM} in combination with the \texttt{jobfarm} framework to parallelize and manage large batches of experiments. A typical job allocation requested \texttt{2} nodes with \texttt{200} tasks (\texttt{2} CPUs per task) and a maximum runtime of 24 hours. The runtime environment was managed with \texttt{Conda}, and all dependencies were contained in a dedicated environment to ensure reproducibility.  

For transparency, the following example job script illustrates the scheduling setup:

\begin{verbatim}
#!/bin/bash
#SBATCH -J JobFarm
#SBATCH --nodes=2
#SBATCH --ntasks=200
#SBATCH --cpus-per-task=2
#SBATCH --time=24:00:00
...
jobfarm start experiment/command_list.txt
\end{verbatim}

This configuration enabled efficient distribution of workloads across multiple CPUs while maintaining consistent software environments.

\subsection{Baselines}\label{subsec:Baselines}

We compare against several established methods for causal discovery:  
\begin{itemize}
    \item \textbf{GIES} \citep{hauser2012characterization}, implemented via the Causal Discovery Toolbox (MIT License)\footnote{\url{https://github.com/FenTechSolutions/CausalDiscoveryToolbox}}.  
    \item \textbf{IGSP} \citep{wang2017permutation}, using the CausalDAG package (3-Clause BSD License)\footnote{\url{https://github.com/uhlerlab/causaldag}}.  
    \item \textbf{DCDI} \citep{brouillard2020differentiable}, with the authors' Python implementation (MIT License).  
    \item \textbf{NODAGS} \citep{sethuraman2023nodags}, using the official implementation (Apache 2.0 License).  
    \item \textbf{LLC} \citep{hyttinen2012learning}, based on the NODAGS code (Apache 2.0 License) with extensions by \citet{Lorch:2024} with code published under MIT License \citep{stadion2024}.  
    \item \textbf{KDS} (\citep{stadion2024}), using the official repository (MIT License)\footnote{\url{https://github.com/larslorch/stadion}}.  
\end{itemize}

\subsection{Results}\label{app:Results}

To evaluate whether our method is significantly better than a baseline method by a relative margin of $5\%$, we perform a paired Wilcoxon signed-rank test on the log-transformed metric ratios. In specific, for any method and data generation process we calculate the Wasserstein-$2$ distance and mean squared error (MSE) in every dimension. For one of our methods and a baseline method we define a paired array $(\operatorname{our}_i, \operatorname{baseline}_i)_{i \in [50 * 10]}$ of metric values, where $50$ is the amount of generated datasets and $10$ is the number of test interventions per dataset. By using the log transform, multiplicative differences become additive. Let $r_i = \log(\text{our}_i) - \log(\text{baseline}_i)$. A 5\% improvement (ours is 5\% smaller) means 
\(\text{our}_i \le 0.95 \cdot \text{baseline}_i\), which is equivalent to $\log(\text{our}_i) - \log(\text{baseline}_i) \le \log(0.95)$.
So we test
\[
H_0: \text{median}(r_i) = 0 \quad \text{vs} \quad H_1: \text{median}(r_i) < \log(0.95).
\]

Equivalently, we can define
\[
d_i = r_i - \log(0.95)
\]
and run a Wilcoxon signed-rank test on \(d_i\) versus 0 (i.e., test whether \(\text{median}(d_i) < 0\)). We run this as well with roles of our method and the baseline method switched. The results for both of our methods are reported in \cref{tab:SKDS Linear significance tests} and \cref{tab:SKDS MPL significance tests}.

As detailed in \cref{tab:SKDS Linear significance tests}, \textbf{SKDS (Linear)} achieves statistically significant improvements (a relative margin of at least $5\%$) over most baselines across all datasets, while performing on par with \textbf{KDS}. \textbf{SKDS (MLP)} also performs strongly and is competitive with all baselines, but shows significant losses in purely linear settings, where simpler models are inherently better suited. This reflects a natural trade-off between flexibility and inference difficulty: more expressive models adapt well to nonlinear data but may underperform when the ground truth is linear.

In the significance tests, the \textbf{SKDS} variants are only outperformed by the baseline \textbf{KDS} and only for one data-generation setting (linear SDEs with Erd\H{o}s--R\'enyi structure) when evaluated with the Wasserstein-$2$ distance. This finding aligns with \cref{prop:SKDS upper bound}, which establishes that the SKDS objective can be upper bounded by the Wasserstein-$2$ distance. In this sense, the observed gap may reflect the tighter sensitivity of Wasserstein-$2$ distance to distributional differences beyond the mean, whereas MSE primarily emphasizes mean accuracy.

\begin{table}[ht]
\centering
\begin{tabular}{lcccccc}
\hline
Baseline / Dataset & SDE-ER & SDE-SF & SCM-ER & SCM-SF & SERGIO-ER & SERGIO-SF \\
\hline
GIES &  / $\ast$ &  / $\ast$ &  / $\ast$ &  /  & $\ast$ / $\ast$ & $\ast$ / $\ast$ \\
IGSP & $\ast$ / $\ast$ &  / $\ast$ & $\ast$ / $\ast$ & \textcolor{red}{$\ast$} / $\ast$ & $\ast$ / $\ast$ & $\ast$ / $\ast$ \\
DCDI & $\ast$ / $\ast$ & $\ast$ / $\ast$ & $\ast$ / $\ast$ & $\ast$ / $\ast$ & \textcolor{red}{$\ast$} /  &  /  \\
LLC & $\ast$ / $\ast$ &  / $\ast$ & $\ast$ / $\ast$ & $\ast$ / $\ast$ & $\ast$ / $\ast$ & $\ast$ / $\ast$ \\
NODAGS & $\ast$ / $\ast$ &  / $\ast$ & $\ast$ / $\ast$ &  / $\ast$ & \textcolor{red}{$\ast$} / $\ast$ &  /  \\
KDS (Linear) &  /  &  /  & $\ast$ / $\ast$ & $\ast$ / $\ast$ & $\ast$ / $\ast$ & $\ast$ / $\ast$ \\
KDS (MLP) & \textcolor{red}{$\ast$} /  & $\ast$ / $\ast$ & $\ast$ / $\ast$ &  /  &  /  &  /  \\
\hline
\end{tabular}
\caption{Significance tests for \textbf{SKDS (Linear)}. Black star indicates our method significantly outperforms the baseline; red star indicates baseline significantly outperforms ours. Each cell shows MSE / Wasserstein results.}
\label{tab:SKDS Linear significance tests}
\end{table}
\begin{table}[ht]
\centering
\begin{tabular}{lcccccc}
\hline
Baseline / Dataset & SDE-ER & SDE-SF & SCM-ER & SCM-SF & SERGIO-ER & SERGIO-SF \\
\hline
GIES & \textcolor{red}{$\ast$} /  &  / $\ast$ & \textcolor{red}{$\ast$} /  & \textcolor{red}{$\ast$} / $\ast$ & $\ast$ / $\ast$ & $\ast$ / $\ast$ \\
IGSP & \textcolor{red}{$\ast$} / $\ast$ &  / $\ast$ & $\ast$ / $\ast$ &  / $\ast$ & $\ast$ / $\ast$ & $\ast$ / $\ast$ \\
DCDI &  / $\ast$ & $\ast$ / $\ast$ & $\ast$ / $\ast$ & $\ast$ / $\ast$ &  / $\ast$ &  / $\ast$ \\
LLC & \textcolor{red}{$\ast$} / $\ast$ &  / $\ast$ & $\ast$ / $\ast$ &  / $\ast$ & $\ast$ / $\ast$ & $\ast$ / $\ast$ \\
NODAGS & \textcolor{red}{$\ast$} / $\ast$ &  / $\ast$ & $\ast$ / $\ast$ &  / $\ast$ &  / $\ast$ &  / $\ast$ \\
KDS (Linear) & \textcolor{red}{$\ast$} / \textcolor{red}{$\ast$} & $\ast$ /  & $\ast$ / $\ast$ & $\ast$ / $\ast$ & $\ast$ / $\ast$ & $\ast$ / $\ast$ \\
KDS (MLP) & \textcolor{red}{$\ast$} / \textcolor{red}{$\ast$} &  /  &  /  &  /  & $\ast$ /  &  /  \\
\hline
\end{tabular}
\caption{Significance tests for \textbf{SKDS (MLP)}. Black star indicates our method significantly significantly outperforms the baseline by at least 5\%; red star indicates baseline significantly outperforms ours. Each cell shows MSE / Wasserstein results.}
\label{tab:SKDS MPL significance tests}
\end{table}

The results in \cref{fig:Results-W2-ER-Full} provide a complementary perspective on model performance. They show that the choice of causal structure does not substantially affect the comparative behavior of the methods. In the linear setting, all approaches achieve similar accuracy with respect to MSE. By contrast, in the gene expression experiments, only \textbf{SKDS (Linear)} remains competitive as a linear model, whereas approaches with nonlinear drift functions exhibit a clear advantage.

\begin{figure}
    \centering
    \includegraphics[width=0.96\linewidth]{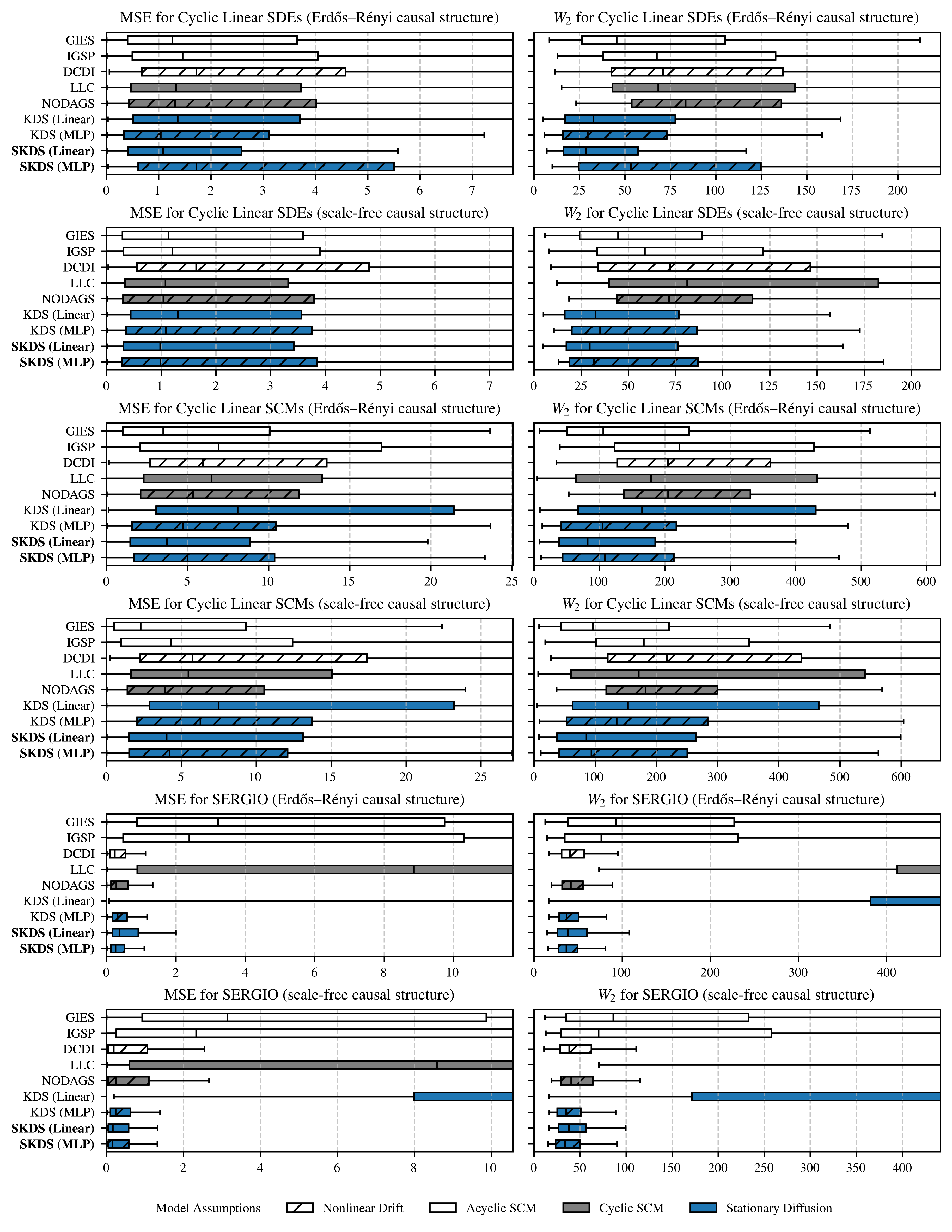}
    \caption{Benchmarking results for $d = 20$ variables. The MSE (left column) and the Wasserstein-$2$ distance (right column) distance MSE were computed from $10$ test interventions on unseen target variables across $50$ randomly generated systems and $6$ different data generating methods (rows). Box plots depict the medians and interquartile ranges (IQR), with whiskers extending to the largest value within $1.5$ times the IQR from the boxes.}
    \label{fig:Results-W2-ER-Full}
\end{figure}

\section{Code}\label{app:Code}

The Stein-type kernel deviation from stationarity can be used as a drop-in replacement for the framework by \citet{stadion2024}. The critical code snippet is to be found in \cref{lst:stein-objective}.

\begin{figure}[ht]
\begin{lstlisting}
def stein_type_kds_operator(h, argnum, hdim):
    def h_out(x, y, *args):
        assert x.ndim == y.ndim == 1
        assert x.shape == y.shape
        z = x if argnum == 0 else y
        f_x = f(z, *args)
        sigma_x = sigma(z, *args)
        grad_h = jax.jacfwd(h, argnums=argnum)(x, y, *args)
        if hdim == 0:
            return h(x,y,*args) * f_x + 0.5 * sigma_x.T @ sigma_x @ grad_h
        if hdim == 1:
            return f_x @ h(x,y,*args) + 0.5 * jnp.trace(sigma_x @ sigma_x.T @ grad_h)
    return h_out

def loss_term(x, y, *args):
    return stein_type_kds_operator(
        stein_type_kds_operator(_kernel, 0, 0), 1, 1
    )(x, y, *args)
\end{lstlisting}
\caption{Nested Stein-type KDS operator defining the custom objective function integrated into the \texttt{stadion} framework~\citep{stadion2024}.}
\label{lst:stein-objective}
\end{figure}

\vfill


\end{document}